\documentclass{article} 
\usepackage{iclr2027_conference,times}
\iclrfinalcopy
\renewcommand{\headrulewidth}{0pt}

\usepackage{amsmath,amsfonts,bm}

\def\eqref#1{equation~\ref{#1}}

\def\1{\bm{1}}

\DeclareMathAlphabet{\mathsfit}{\encodingdefault}{\sfdefault}{m}{sl}
\SetMathAlphabet{\mathsfit}{bold}{\encodingdefault}{\sfdefault}{bx}{n}

\DeclareMathOperator*{\argmin}{arg\,min}

\usepackage[colorlinks=true,linkcolor=blue,citecolor=blue,urlcolor=blue]{hyperref}
\usepackage{hyperref}
\usepackage{url}
\usepackage{wrapfig}
\usepackage{booktabs}
\usepackage{amsmath}
\usepackage{amssymb}
\usepackage{mathtools}
\usepackage{amsthm}
\usepackage{comment}
\theoremstyle{plain}
\newtheorem{theorem}{Theorem}[section]

\newtheorem{lemma}[theorem]{Lemma}
\newtheorem{example}[theorem]{Example}
\newtheorem{corollary}[theorem]{Corollary}
\theoremstyle{definition}
\newtheorem{definition}[theorem]{Definition}

\theoremstyle{remark}
\newtheorem{remark}[theorem]{Remark}
\usepackage{tikz}
\usetikzlibrary{arrows.meta,backgrounds}

\definecolor{gdblue}{HTML}{E68613}       
\definecolor{muonpurple}{HTML}{377EB8} 
\definecolor{noisegreen}{HTML}{2E8B57}   

\definecolor{modeone}{HTML}{C44E52}      
\definecolor{modetwo}{HTML}{8172B2}      
\definecolor{modethree}{HTML}{4C4C4C}    

\definecolor{textgray}{HTML}{333333}
\usepackage{graphicx}
\title{From Parallel to Sequential: Small-Batch Noise Restores Simplicity Bias of Muon}

\author{Tom Jacobs \& Rebekka Burkholz \\
CISPA Helmholtz Center for Information Security, Germany \\
\texttt{\{tom.jacobs,burkholz\}@cispa.de} 
}

\newcommand{\muongdschematic}{%
\begin{tikzpicture}[
  font=\sffamily,
  >=Latex,
  line width=0.6pt,
  head/.style={font=\sffamily\bfseries\small, text=textgray},
  card/.style={rounded corners=2pt, minimum height=1.18cm, align=center,
               inner xsep=7pt, inner ysep=4pt},
  regime/.style={card, minimum width=2.55cm, draw=##1, fill=##1!7},
  update/.style={card, minimum width=3.50cm, draw=##1!80!black, fill=white},
  magbox/.style={card, minimum width=2.05cm, draw=##1!80!black, fill=##1!4},
  result/.style={card, minimum width=4.80cm, draw=##1!80!black, fill=##1!4},
  flow/.style={->, draw=textgray!60, line width=0.8pt},
  tinylabel/.style={font=\sffamily\footnotesize, text=textgray},
  maglabel/.style={font=\sffamily\footnotesize\bfseries, text=textgray,
                   fill=white, inner sep=0.8pt},
  pics/magnitudes/.style args={##1/##2/##3}{code={
    \draw[->, textgray!45] (-0.10,-0.32) -- (1.42,-0.32);
    \fill[modeone]   (0.06,-0.32) rectangle (0.31,##1);
    \fill[modetwo]   (0.52,-0.32) rectangle (0.77,##2);
    \fill[modethree] (0.98,-0.32) rectangle (1.23,##3);
  }},
  pics/times/.style args={##1/##2/##3}{code={
    \draw[->, textgray!45] (0,-0.3) -- (2.15,-0.3) node[left=59pt, above=-2pt, font=\tiny] {$t$};
    \draw[textgray!12] (0,0.3) -- (2.00,0.3);
    \draw[textgray!12] (0,0.00) -- (2.00,0.00);
    \draw[textgray!45] (0,-0.3) -- (2.00,-0.3);
    \fill[modeone]   (##1,0.3) circle (2.7pt);
    \fill[modetwo]   (##2,0.00) circle (2.7pt);
    \fill[modethree] (##3,-0.3) circle (2.7pt);
  }}
]

\node[head] at (0,-0.3) {setting};
\node[head] at (3.20,-0.3) {update};
\node[head] at (6.25,-0.3) {$|G|$};
\node[head] at (10.00,-0.3) {feature activation};

\node[regime=gdblue] (gd) at (0,-1.35)
  {\textbf{GD}\\[-1pt]\footnotesize any $B$};
\node[update=gdblue] (gdu) at (3.20,-1.35) {$-G$};
\node[magbox=gdblue] (gdm) at (6.25,-1.35) {};
\pic at ([xshift=-0.70cm]gdm.center) {magnitudes=0.26/0.05/-0.15};
\node[maglabel] at ([yshift=0.39cm]gdm.center) {kept};
\node[result=gdblue] (gdr) at (10.00,-1.35) {};
\node[tinylabel] at ([xshift=-1.27cm,yshift=0.41cm]gdr.center) {$\theta_0\!\downarrow$};
\node[tinylabel] at ([xshift=1.27cm,yshift=0.41cm]gdr.center) {$\theta_0\!\uparrow$};
\pic at ([xshift=-2.02cm,yshift=-0.10cm]gdr.center) {times=0.35/1.02/1.68};
\pic at ([xshift=0.15cm,yshift=-0.10cm]gdr.center) {times=0.82/1.02/1.22};

\node[regime=muonpurple] (ml) at (0,-3.02)
  {\textbf{Muon}\\[-1pt]\footnotesize $B\!\uparrow$};
\node[update=muonpurple] (mlu) at (3.20,-3.02) {$-\operatorname{Sign}(G)$};
\node[magbox=muonpurple] (mlm) at (6.25,-3.02) {};
\pic at ([xshift=-0.70cm]mlm.center) {magnitudes=0.10/0.10/0.10};
\node[maglabel] at ([yshift=0.39cm]mlm.center) {flattened};
\node[result=muonpurple] (mlr) at (10.00,-3.02) {};
\node[tinylabel] at ([yshift=0.41cm]mlr.center) {$\theta_0\!\downarrow\ \text{or}\ \theta_0\!\uparrow$};
\pic at ([xshift=-1.02cm,yshift=-0.10cm]mlr.center) {times=1.02/1.02/1.02};

\node[regime=noisegreen] (ms) at (0,-4.69)
  {\textbf{Muon}\\[-1pt]\footnotesize $B\!\downarrow,\ \sigma\!\uparrow$};
\node[update=noisegreen] (msu) at (3.20,-4.69)
  {$-\mathbb E\!\left[\operatorname{Sign}(G+\xi)\right]$\\[-1pt]
   \footnotesize $\approx-c(\sigma)G$};
\node[magbox=noisegreen] (msm) at (6.25,-4.69) {};
\pic at ([xshift=-0.70cm]msm.center) {magnitudes=0.26/0.05/-0.15};
\node[maglabel] at ([yshift=0.39cm]msm.center) {restored};
\node[result=noisegreen] (msr) at (10.00,-4.69) {};
\node[tinylabel] at ([xshift=-1.27cm,yshift=0.41cm]msr.center) {$\theta_0\!\downarrow$};
\node[tinylabel] at ([xshift=1.27cm,yshift=0.41cm]msr.center) {$\theta_0\!\uparrow$};
\pic at ([xshift=-2.02cm,yshift=-0.10cm]msr.center) {times=0.35/1.02/1.68};
\pic at ([xshift=0.15cm,yshift=-0.10cm]msr.center) {times=0.82/1.02/1.22};

\foreach \a/\b in {gd/gdu,gdu/gdm,gdm/gdr,ml/mlu,mlu/mlm,mlm/mlr,ms/msu,msu/msm,msm/msr}
  \draw[flow] (\a.east) -- (\b.west);

\begin{scope}[on background layer]
  \draw[textgray!12] (-1.35,-2.18) -- (12.50,-2.18);
  \draw[textgray!12] (-1.35,-3.85) -- (12.50,-3.85);
\end{scope}

\end{tikzpicture}
}

\begin{document}

\maketitle

\begin{abstract}%
Simplicity bias, which can arise from sequential feature learning during training, has been proposed as a general explanation for the success of deep learning. An important driver of this phenomenon is saddle-to-saddle behavior, in which the gradient flow trajectory remains close to a sequence of saddle manifolds while learning one feature at a time. However, Muon, a state-of-the-art optimizer based on spectral descent, does not exhibit such dynamics in the full-batch setting and instead learns features in parallel. 
How is it then that Muon is so successful in deep learning settings? We show that stochastic gradient noise fundamentally changes this behavior: small-batch noise restores sequential feature learning, providing a mechanism for simplicity bias that is absent from full-batch Muon. We explain this transition through stochastic smoothing of the spectral descent operator underlying Muon. When gradients are small relative to the noise scale, the expected spectral update becomes approximately proportional to the actual gradient, which restores gradient-flow-like dynamics. We formalize this mechanism for deep linear models by deriving the effective continuous dynamics of noisy spectral descent. Experiments on low-rank matrix reconstruction, vision transformers, and large language models (LLMs) support these findings. 
\end{abstract}

\section{Introduction}

Deep neural networks are heavily overparameterized: Many different predictors
can fit the training data, yet, among the many existing solutions optimization often selects ones that
generalize well. 
One proposed explanation is a simplicity bias: optimization preferentially learns simpler predictors. Sequential feature learning can produce such a bias when simpler features are learned before more complex ones.
At the level of the optimization trajectory, this behavior is often described
through saddle-to-saddle dynamics, where the trajectory moves through a
sequence of increasingly complex regions of parameter space
\citep{pesme2023saddletosaddle,NEURIPS2022_7eeb9af3, saxe2014exactsolutionsnonlineardynamics}. 

\begin{figure}[t]
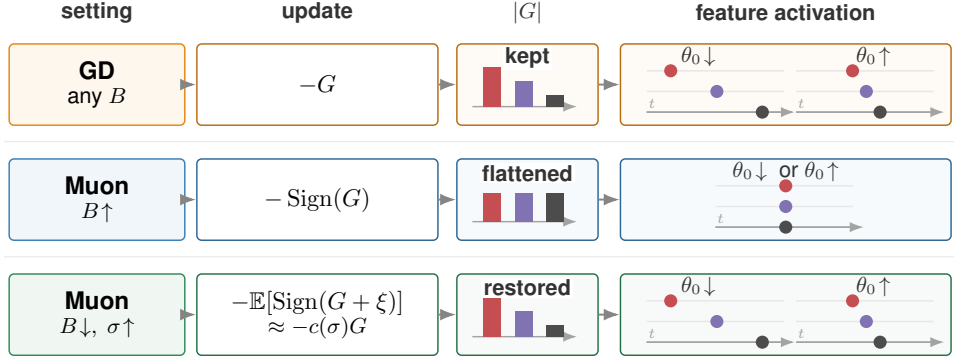

     \centering
     \resizebox{0.91\linewidth}{!}{\muongdschematic}
     \caption{\textbf{Batch size $B$ controls the simplicity bias of Muon.}
Gradient descent preserves mode-dependent gradient magnitude ($|G|$), producing sequential feature activation from small initialization. In contrast, full-batch Muon ($B \uparrow$) flattens these magnitudes and learns modes in parallel, whereas small-batch noise ($\sigma \uparrow$) smooths the spectral sign operator and restores initialization-dependent sequential learning.}\label{figure : schematic introduction}
\end{figure}

This becomes less clear once we go beyond gradient descent. Muon \cite{jordan2024muon},
a state-of-the-art optimizer based on spectral descent \citep{bernstein2025modular},
can behave very differently from gradient flow
\citep{Dragutinovic2026ToUO,Rhee2026MuonLB}. In simple matrix-factorization problems,
gradient flow learns singular directions sequentially, whereas full-batch Muon learns them
in parallel. This is surprising: if sequential feature learning is an important source of
simplicity bias, why does an optimizer that removes this behavior work so well in modern
deep learning?

In this work, we show that the missing piece in this consideration is small-batch noise. The behavior of Muon changes qualitatively once the gradients become stochastic. While full-batch spectral descent learns different features in parallel, small-batch spectral descent can recover a clear ordering in which stronger features are learned before weaker ones.
By learning dominant singular directions first, these dynamics favor low-rank intermediate predictors, recovering a mechanism for simplicity bias that is absent from deterministic Muon.

Muon applies a matrix sign, or spectral normalization, to the gradient. In the deterministic setting, this removes much of the information contained in the gradient magnitude and causes different singular directions to evolve at the same rate. Our main insight is that small-batch noise smoothens this nonlinear sign operation. When the gradient is small relative to the noise scale, the expected normalized update becomes approximately proportional to the original gradient. As a result, the  gradient magnitude enters the dynamics and larger target features are learned faster. We schematically illustrate this in Figure \ref{figure : schematic introduction}.

\begin{figure}[ht]
    \centering
   \includegraphics[width = 0.99\linewidth]{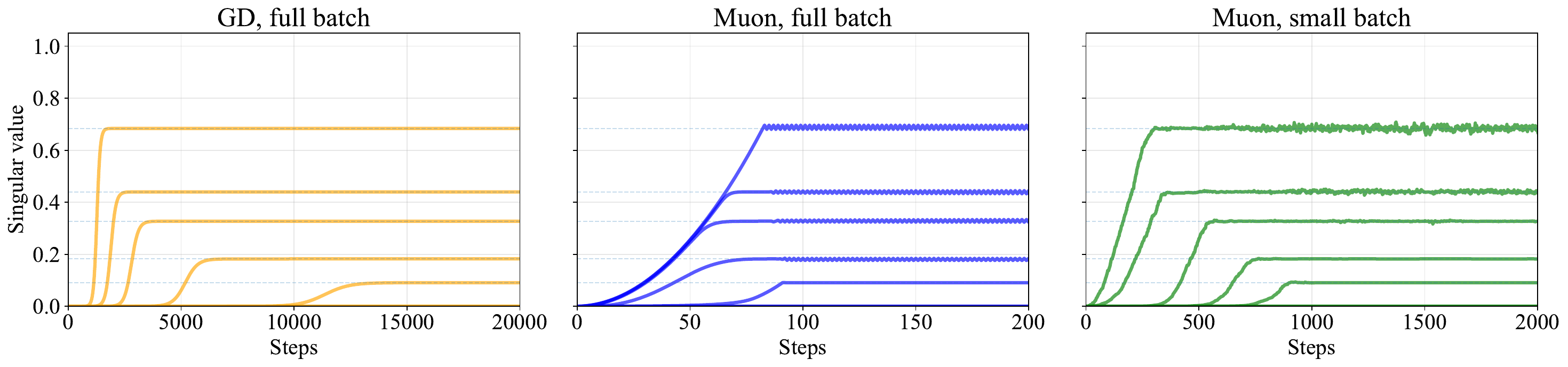}
    \caption{For low-rank matrix reconstruction with whitened data, gradient flow learns singular directions sequentially, while full-batch Muon (without momentum) learns them in parallel. Small-batch noise restores sequential feature learning, thereby giving a mechanism for simplicity bias in Muon.}
    \label{fig: main}
\end{figure}

We derive this mechanism explicitly for deep linear matrix factorization. For whitened data, we first obtain the singular-value dynamics of deterministic spectral flow (the flow proxy for full-batch Muon) and show that it produces parallel feature learning. We then introduce loss-dependent stochastic gradient noise and derive the corresponding expected spectral flow (i.e. the flow proxy for small-batch Muon). The resulting smoothed dynamics exhibit a gradient dependent growth rate, which leads to separation between singular modes and sequential feature learning. Increasing depth further amplifies this trend. The expected spectral flow corresponding to Muon retains this leading-order smoothing mechanism even for symmetric heavy-tailed noise and even with infinite variance, which could substantially change dynamics of SGD.
For the diagonal restriction of the dynamics, we further show that depth changes the implicit bias: under infinitesimal initialization, the induced regularization is of $L_{1+1/L}$-type at depth $L$, whereas gradient flow retains an $L_1$ bias. The corresponding mirror potential also remains convex at arbitrary depth and is related to the steepest mirror flow formulation by \cite{jacobs2026never}. This contrasts with gradient flow, whose corresponding potential loses convexity beyond depth two in the balanced limit.

Our theory implies that batch size can determine whether Muon exhibits the sequential feature learning that favors low-rank intermediate predictors. This is visible, for instance, in the low-rank reconstruction example in Figure~\ref{fig: main}: Full-batch Muon without momentum (also known as spectral descent) learns the relevant singular directions in parallel, whereas decreasing the batch size produces sequential dynamics. 
Beyond this analytically tractable setting, we investigate whether smaller batches also favor simpler learned representations and find lower effective ranks in both vision and language models.

\paragraph{Contributions.}
\begin{itemize}
   \item \textbf{Small-batch noise restores a mechanism for simplicity bias in Muon.}
We derive effective spectral-flow dynamics in Theorem \ref{theorem : spectral smooth flow} showing how stochastic smoothing makes the matrix-sign update approximately gradient proportional, recovering sequential feature learning from the parallel dynamics of full-batch Muon (Theorem \ref{theorem: spectral flow}), with depth amplifying the separation between feature-learning times.

\item \textbf{Heavy-tailed noise.}
We show that the same smoothing mechanism extends to heavy-tailed noise, even with infinite variance as in Example \ref{example : heavy tail cauchy}, preserving the recovery of sequential learning despite noise distributions that substantially alter SGD dynamics.

    \item \textbf{Implicit regularization of smoothed spectral dynamics.}
In Section \ref{section : implicit regularization}, for the diagonal restriction of the aligned factorization dynamics, we
characterize the implicit bias of the smoothed sign flow underlying
Muon in Theorem \ref{lemma : implicit bias two layer}. The resulting depth-dependent mirror potential remains convex at
arbitrary depth, unlike the corresponding potential for gradient flow.

\item \textbf{A link between batch size and learned complexity.}
    Experiments in Section \ref{section : empirical validation} show that smaller batches yield representations with lower effective rank in both vision and language models, consistent with the preference for simpler representations suggested by our theory.
\end{itemize}

\section{Related Work}
\label{sec:related-work}

\paragraph{Sequential feature learning.}
A recurring observation in neural-network optimization is that features are often acquired sequentially rather than in parallel. In deep linear and matrix-factorization models, gradient dynamics can exhibit mode-dependent learning times and preferential recovery of dominant spectral components
\citep{saxe2014exactsolutionsnonlineardynamics,NEURIPS2023_cad2fd66, Huh2021TheLS, Andriushchenko2022SGDWL, Gidel2019ImplicitRO, zhao2026gradient}.
Related sequential dynamics have been characterized through trajectories that move between increasingly complex saddle manifolds in diagonal linear networks \citep{pesme2023saddletosaddle}, and under structured assumptions for shallow ReLU networks \citep{NEURIPS2022_7eeb9af3}. Complementary work distinguishes several mechanisms through which feature learning can arise in graph neural networks \citep{mustafa2023are,mustafa2024gate,Mustafa2024DynamicRF} and in linear models \citep{Xu2024ThreeMO}. More broadly, incremental learning has been linked to the geometry induced by mirror flows \citep{Berthier2026IncrementalLI}. 
Most closely related to our setting, recent work proposes saddle-to-saddle dynamics as an explanation for simplicity bias across architectures \citep{zhang2026saddletosaddle}, while \citep{Dragutinovic2026ToUO, Rhee2026MuonLB} demonstrate that this behavior can disappear under full-batch Muon, where singular modes are learned substantially more in parallel. Our contribution addresses the resulting gap: We show that the absence of sequential feature learning is specifically a deterministic phenomenon, and that small-batch stochasticity can restore a strong ordering of feature-learning under Muon.

\paragraph{Muon, steepest descent, and optimization geometry.}
Muon can be viewed through steepest descent in a spectral-norm geometry, with matrix-valued gradients transformed by an approximate polar or matrix-sign operation \citep{bernstein2025modular}. This places it within a broader literature relating optimization geometry to implicit bias \citep{gunasekar2020characterizing}. Sign-based and non-Euclidean steepest-descent methods have been studied in homogeneous networks, transformers, and separable classification
\citep{Balles2020TheGO,tsilivis2025flavors,li2025on,fan2025implicit}, while mirror-flow analyses characterize how parameterization and geometry can favor different representations
\citep{pesme2024implicit,Jacobs2026ImplicitBO,jacobs2026never, woodworth2020kernel, vaškevičius2019implicitregularizationoptimalsparse, tarzanagh2024transformerssupportvectormachines, arora2019implicitregularizationdeepmatrix, liang2025implicit, Li2022ImplicitBO, azulay2021implicitbiasinitializationshape, chou2024induce, Chou2021MoreIL, Chou2023RobustIR, sun2022mirror, domin2024from, Liang2026ImplicitBO}. Related work also studies smoothed sign descent from a mirror-descent perspective \citep{Wang2024AMD}. These results primarily concern convergence, margin or implicit bias of the limiting predictor. Instead, our focus is on the transient dynamics that determine which features emerge first. Recent modifications and analyses of spectral optimization further reveal that additional modifications can be beneficial including spectral clipping and general Schatten norms \citep{Jiang2026EnhancingLT,Pethick2026WhenTU,Yang2026ConvergenceOS}, while random or inverted spectral transformations can remain surprisingly competitive with the standard Muon update \citep{Shumaylov2026MuonIN}.
This suggests that Muon’s empirical success may depend less on the precise deterministic spectral transformation itself than on how that nonlinear transformation interacts with stochastic gradients during training.

\paragraph{Stochasticity and nonlinear gradient normalization.}
Small-batch gradient noise is generally state dependent and coupled to both the data distribution and the current loss, rather than behaving as an independent additive perturbation \citep{wojtowytsch2021stochastic, MonzioCompagnoni2025UnbiasedAS}. Such stochasticity can change the solution selected by optimization, including the implicit bias of diagonal and rank-one linear networks
\citep{pesme2021implicit,lyu2023implicit}, and injected noise can itself induce regularization in overparameterized models \citep{pmlr-v206-orvieto23a}. Different joint limits of learning rate, noise scale, and training time can also yield qualitatively different effective dynamics \citep{Shalova2024SingularlimitAO}. Particularly relevant to normalized methods, high-dimensional analyses of signSGD quantify how the sign operation compresses gradient noise and modifies effective preconditioning \citep{Xiao2024ExactRC}, while recent Muon-specific results show that spectral normalization can improve convergence under heavy-tailed gradient noise \citep{Hubler2026FreeHL}.

\section{Preliminaries}\label{section : preliminaries}

\paragraph{Matrix factorization.}
Deep linear networks provide a tractable setting in which gradient-flow dynamics can be studied explicitly. Let
\(X \in \mathbb{R}^{d \times n}\) and
\(Y \in \mathbb{R}^{d \times n}\) denote the input and output data
matrices. For a two-layer linear network with parameters 
\(U,V \in \mathbb{R}^{d \times d}\), we consider the empirical loss
$
\mathcal{L}(U,V)
=
\frac{1}{2n}\|UVX-Y\|_{\mathrm F}^2.
$
We assume that,
\(XX^\top/n=I_d\) and write
\(C_{yx}=YX^\top/n=P\Lambda Q^\top\) for the singular value
decomposition of the input--output cross-covariance.

If the factors are initialized in the corresponding singular-vector
basis, the different singular modes remain decoupled under gradient
flow. We further assume balanced initialization $U_0^\top U_0 = V_0 V_0^\top$, meaning that the two
factors have equal magnitude along each mode and we assume that the solution is realizable. Let \(\theta_i\) denote
the \(i\)-th singular value of the end-to-end matrix \(UV\), and let
\(\lambda_i\) be the corresponding singular value of the cross-covariance matrix \(C_{yx}\).

\begin{theorem}[Section 1 in \citep{saxe2014exactsolutionsnonlineardynamics}]\label{theorem : gradient flow two layer}
\label{thm:matrix-factorization}
Under whitened data, singular-vector alignment, and balanced
initialization, each singular mode evolves independently. Up to a
constant rescaling of time, its gradient-flow dynamics are
\[
\frac{d\theta_i}{dt}
=
\theta_i\bigl(\lambda_i-\theta_i\bigr),
\qquad
\theta_{i,0}=\theta_{i,\mathrm{init}}, \text{ with solution } \theta_{i,t}
=
\frac{\lambda_i}{
1+
\left(
\lambda_i/\theta_{i,\mathrm{init}}-1
\right)e^{-\lambda_i t}
},
\]
when \(\lambda_i>0\) and \(\theta_{i,\mathrm{init}}>0\).
\end{theorem}
The same description extends to an \(L\)-layer linear network
\(W_L\cdots W_1\). If the layers are aligned and balanced across depth,
then the end-to-end map singular value \(\theta_i\) satisfies
\[
\frac{d\theta_i}{dt} 
=
L\theta_i^{\,2-2/L}
\bigl(\lambda_i-\theta_i\bigr).
\]
Thus, depth changes the effective learning rate of each mode but does not
destroy the mode-wise decoupling under the stated assumptions.

\paragraph{Batch noise.}
When we train with small batches, the effective update is a stochastic gradient $g$ instead of the full gradient $\nabla \mathcal{L}$. In our dynamical model of Muon, we want to study the effect of such a noise on the expected dynamics. Concretely, this type of noise typically is assumed to have the following properties:
1) The noise is centered and its covariance is bounded by the loss $\mathcal{L}$ for mean squared error type of losses \citep{wojtowytsch2021stochastic}.
2) The gradient estimates $g$ are independent when conditioned on the current parameters $\theta$.
3) The noise is coupled across layers \citep{pesme2021implicit}.
4) The noise has low rank.
We consider a diagonal noise structure that is isotropic, so not of low rank. Concretely, we have the estimate $g_{UV}= \nabla_{UV} \mathcal{L}(UV) + \xi$, with $\xi \sim N(0, \sigma^2\mathcal{L}I)$, which satisfies the first $3$ assumptions.

\paragraph{Steepest descent.}
The Muon update, without momentum, is a steepest descent step with respect to the spectral norm \citep{bernstein2025modular}. The update itself is the matrix sign operator as given by Definition \ref{definition : matrix sign opperator}.

\begin{definition}[Matrix sign operator]\label{definition : matrix sign opperator}
    For a matrix $G \in \mathbb{R}^{d \times d}$ and its singular value decomposition (SVD) $(A, B, \Lambda)$ fulfilling $G = A \Lambda B^T$, the sign operator $\text{Sign} : \mathbb{R}^{d \times d} \rightarrow  \mathbb{R}^{d \times d}$ is defined as:
    $
        \text{Sign}(G) : = A \operatorname{diag}\!\left(
\mathbf{1}_{\{\Lambda_{ii}>0\}}
\right) B^T.
    $
\end{definition}

For a fixed learning rate $\eta > 0$, this gives the training update $U_{k+1} = U_k -  \eta \ \text{Sign}(\nabla_U \mathcal{L}(U_kV_k))$, $V_{k+1} = V_k -  \eta \ \text{Sign}(\nabla_V \mathcal{L}(U_kV_k))$. Our goal is to discover the resulting dynamics and to compare them to gradient flow ($\eta \rightarrow 0$). Note that under our noise assumption, stochasticity leaves gradient flow unchanged in expectation, since $\mathbb{E}[g_{UV}] = \nabla_{UV}\mathcal{L}$. 
Hence, any noise induced change in the learning dynamics arises specifically from Muon's nonlinear $\operatorname{Sign}$ transformation.

\section{Spectral Dynamics}\label{section : main result spectral dynamics}
We provide two simple solutions of the spectral dynamics underlying Muon, mirroring known results for gradient flow \cite{saxe2014exactsolutionsnonlineardynamics}. To start, we derive the simple full batch case. Next, we present our main result characterizing the expected dynamics under noisy gradient which leads to sequential feature learning. Finally, we extend our arguments also to heavy-tailed noise \citep{Simsekli2019ATA}.

\paragraph{Full-batch dynamics.}
We first isolate the deterministic effect of the spectral flow, which removes the gradient-magnitude information responsible for sequential feature learning. This flow is given by $d U_t = - \text{Sign} (\nabla_U \mathcal{L}(U_tV_t)) dt$, $d V_t = - \text{Sign} (\nabla_V \mathcal{L}(U_tV_t))dt$.
\begin{theorem}[Spectral flow]\label{theorem: spectral flow}
     For whitened data, the spectral gradient flow dynamics of a singular value $\theta$ of the matrix product $UV$ is:
    \begin{equation*}
        d \theta_t =  \sqrt{\theta_t} \ \text{Sign} \ (\lambda - \theta_t) dt, \qquad \theta_0 = \theta_{\text{init}}.
    \end{equation*}
This has a closed-form solution:
\begin{equation*}
\theta_t = 
\left( \sqrt{\lambda} + \operatorname{Sign}\!\left( \sqrt{\theta_{\mathrm{init}}}-\sqrt{\lambda} \right) \left( \left| \sqrt{\theta_{\mathrm{init}}}-\sqrt{\lambda} \right| -\frac{t}{2} \right)_+ \right)^2\end{equation*}
\end{theorem}
Proof. See Appendix \ref{section : expected flows}. $\square$

The full-batch result thus removes the mode-dependent rates underlying sequential learning as shown in Figure \ref{fig: main}. 

\paragraph{Recovering sequential feature learning with noise.}
We now show that stochastic smoothing restores the direct dependence of the gradient on singular values in expectation.

\begin{theorem}[Smoothened spectral flow]\label{theorem : spectral smooth flow}
Assume whitened data, singular-vector alignment, and balanced initialization. Let $\xi_\tau\sim N\left(0,\sigma^2\mathcal{L}_\tau I\right).$ Then the expected dynamics of the \(i\)-th singular value satisfy 
\[
\frac{d\theta_i}{d\tau}
=
\sqrt{\theta_i}\, \mathbb{E}_{N(0,\sigma^2 \mathcal{L_{\tau}})}\left[\text{Sign}(\lambda_i -\ \theta_i + \xi_{\tau})\right]
= \sqrt{\theta_i}\,
\operatorname{erf}\left(
\frac{\lambda_i-\theta_i}
{\sqrt{2\sigma^2\mathcal{L}_\tau}}
\right).
\]
For large \(\sigma\), under the time change $
t(\tau)
=
\sqrt{\frac{2}{\pi\sigma^2}}
\int_0^\tau \frac{ds}{\sqrt{\mathcal{L}_s}},
$
the leading-order dynamics are
\begin{equation}\label{equation : full batch muon}
\frac{d\theta_i}{dt}
=
\sqrt{\theta_i}\left(\lambda_i-\theta_i\right), \text{ with solution } \theta_t
=
\lambda
\tanh^2\left(
\frac{\sqrt{\lambda}}{2}t
+
\operatorname{artanh}
\sqrt{\frac{\theta_{\mathrm{init}}}{\lambda}}
\right),
\end{equation}
when \(0<\theta_{\mathrm{init}}<\lambda\). 
Moreover, the loss of the flow in (\ref{equation : full batch muon}) satisfies $\mathcal{L}_{\tau} \leq \left(\sqrt{\mathcal{L}_0} - \frac{C}{2} \tau \right)^2_+$ with $C = \min_i\sqrt{\theta_{i,0}} \sqrt{\frac{2}{\pi \sigma^2}} $. Therefore, $\mathcal{L_{\tau}} = 0$ for $\tau \geq 2\sqrt{\mathcal{L}_{0}} /C $, showing convergence in finite time.

\end{theorem}
\textit{Proof sketch. 
Singular-vector alignment reduces the dynamics to scalar modes, and taking the expectation of each noisy sign update gives the error-function expression.
The second step follows from a first order Taylor approximation, which remains valid for all $t \geq 0$ (see Remark \ref{remark : noise approximation}). Then the ODE can be solved using a substitution. Finally, the loss bound follows from a Gr\"onwall argument. The full proof is given in Appendix \ref{section : expected flows}. $\square$
}

According to Theorem \ref{theorem : spectral smooth flow}, larger singular values are learned earlier. Still, the flow lacks the pronounced delays observed in Figure \ref{fig: main}, possibly because our isotropic noise model omits the low-rank structure of mini-batch gradients. Increasing depth strengthens this separation (Corollary \ref{corollary : deep spectral smooth}).

\begin{wrapfigure}{r}{0.5\textwidth} 
        \includegraphics[width=0.99\linewidth]{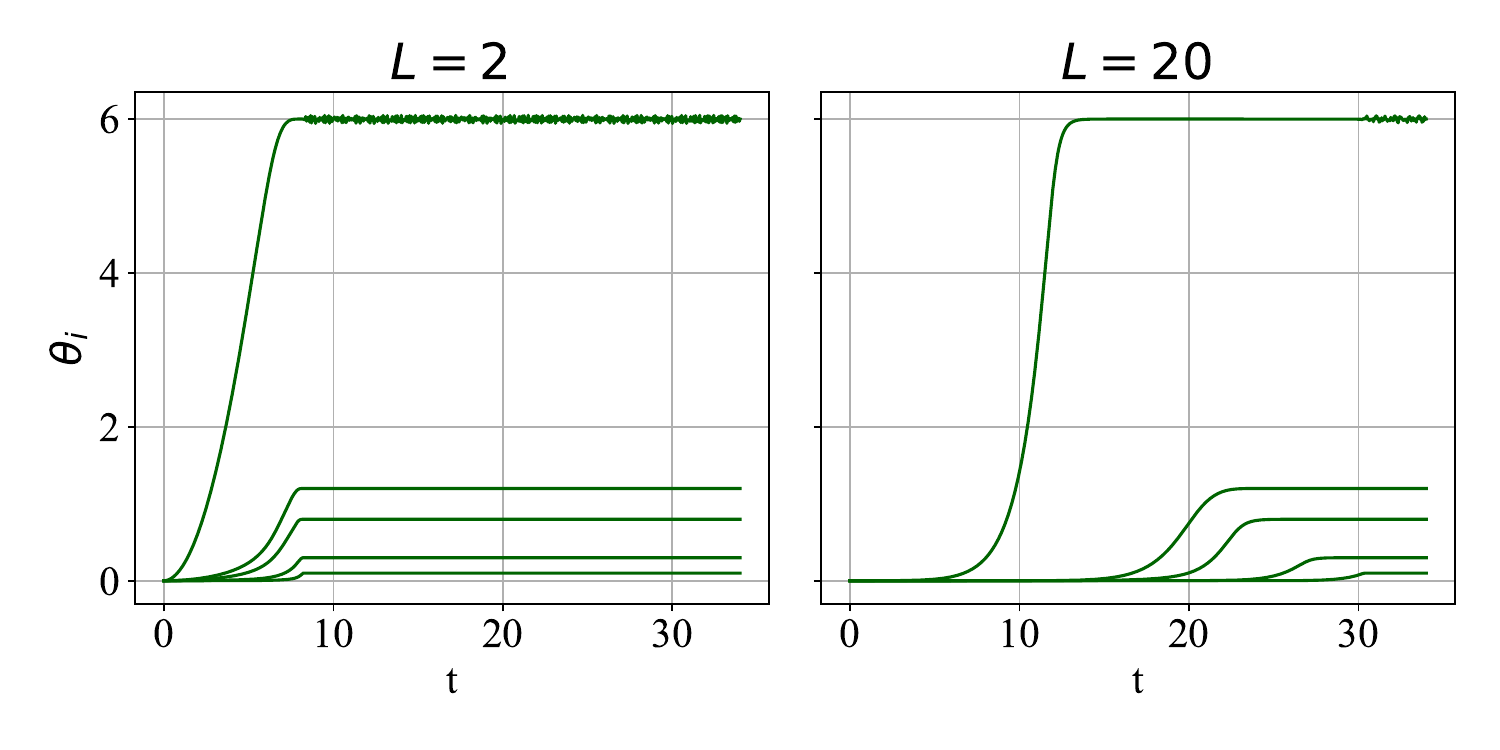}
        \caption{Simulation of the smoothened flows for depth $L = 2$ and $20$. For higher depth we observe sequential feature learning.}
        \label{fig:depth smoothened}
\end{wrapfigure}

\begin{corollary}[Smoothened deep spectral flow]\label{corollary : deep spectral smooth}
     Consider the same setting as Theorem \ref{theorem : spectral smooth flow} with the deep matrix product of $L$ layers. The expected smoothened spectral gradient flow dynamics after the same time rescaling are:
    \begin{equation*}
        d \theta_t =  \theta_t^{1-1/L} (\lambda - \theta_t) dt, \qquad \theta_0 = \theta_{\text{init}}.
    \end{equation*}
\end{corollary}

As $L\rightarrow\infty$, Corollary \ref{corollary : deep spectral smooth} approaches the two-layer gradient-flow dynamics of Theorem \ref{theorem : gradient flow two layer}. Thus, noise and depth together recover increasingly pronounced sequential feature learning, as illustrated in Figure \ref{fig:depth smoothened}. Moreover, near small initialization, the mobility $\theta^{1-1/L}$ is less suppressive than the gradient-flow mobility $\theta^{2-2/L}$, yielding faster initial growth for all $L\geq2$ up to a constant time rescaling.

\begin{remark}\label{remark : noise approximation}
    Note that our approximation of the expected flow crucially relies on the noise construction: The proportionality to $\mathcal{L}$ makes it possible to bound the $\text{erf}$ approximation at all time. If we had considered a fixed isotropic scale, i.e. $N(0, \sigma^2 I)$, this would have only been possible with significantly high noise whose variance would have dominated the dynamics.
\end{remark}

\paragraph{Heavy-tailed noise.}
The Gaussian assumption is not needed for the expected-flow approximation.
It is sufficient that the noise distribution is symmetric and sufficiently
regular near the origin. 
This highlights another advantage of Muon over SGD, whose
dynamics can change substantially under heavy-tailed gradient noise
\citep{Simsekli2019ATA,pmlr-v139-gurbuzbalaban21a}. Since spectral descent
applies a bounded sign transformation, the approximation remains valid
for noise even with infinite variance. 
This is particularly relevant for Muon in LLM training, where heavy-tailed gradient noise has been observed. Recent work establishes convergence advantages for Muon under such noise \citep{Hubler2026FreeHL, Yu2026SignBasedOA}.

\begin{lemma}\label{lemma:noise}
Let the noise $\xi_i$ have a symmetric distribution with cumulative distribution function
$F_{\xi_i}$ and a density $f_{\xi_i}$ that is twice continuously differentiable
in a neighborhood of the origin with $f_{\xi_i}(0) > 0$. Then, for $a_i=\lambda_i-\theta_i$, the expected update as in Theorem \ref{theorem : spectral smooth flow} is:
\begin{equation*}
    \mathbb{E}_{\xi_i}\!\left[\text{Sign}(a_i+\xi_i)\right]
    = 2F_{\xi_i}(a_i)-1
    = 2f_{\xi_i}(0)a_i
      + \frac{f_{\xi_i}''(0)}{3}a_i^3
      + o(a_i^3).
\end{equation*}
\end{lemma}

\begin{proof}
By symmetry, $\mathbb{E}_{\xi_i}[\text{Sign}(a_i+\xi_i)]
=1-2F_{\xi_i}(-a_i)
=2F_{\xi_i}(a_i)-1.$ This follows from a Taylor expansion  and the fact that symmetry
implies $f_{\xi_i}'(0)=0$.
\end{proof}

\begin{remark}
More explicitly, let $\xi_i=\sigma Z_i$, where $Z_i$ has a symmetric density
$f_Z$. Then
\begin{equation*}
    \mathbb{E}\!\left[\text{Sign}(a_i+\sigma Z_i)\right]
    =
    \frac{2f_Z(0)}{\sigma}a_i
    +\frac{f_Z''(0)}{3\sigma^3}a_i^3
    +o(\sigma^{-3}).
\end{equation*}
Thus, as $\sigma\to\infty$, the rescaled expected dynamics recover the same gradient-proportional flow as in the Gaussian case.
\end{remark}

\begin{example}\label{example : heavy tail cauchy}
Consider zero centered heavy-tailed Cauchy noise with scale parameter $\gamma>0$. Then
\begin{equation*}
    \mathbb{E}_{\xi_i}\!\left[\text{Sign}(a_i+\xi_i)\right]
    =
    \frac{2}{\pi}\arctan\!\left(\frac{a_i}{\gamma}\right)
    =
    \frac{2}{\pi}\frac{a_i}{\gamma}
    -\frac{2}{3\pi}\left(\frac{a_i}{\gamma}\right)^3
    +O\!\left(\frac{a_i^5}{\gamma^5}\right).
\end{equation*} 
Thus, even infinite-variance Cauchy noise preserves the smoothing mechanism underlying sequential feature learning.
\end{example}

Since the spectral sign update is bounded even under infinite-variance noise, stochastic fluctuations vanish as $\eta\to0$ at fixed $\sigma$, yielding a deterministic expected flow. This contrasts with the L\'{e}vy-driven dynamics that can arise for SGD under heavy-tailed gradient noise \citep{Nguyen2019FirstET}. When $\sigma\to\infty$ jointly with $\eta\to0$, imposing $\sigma\eta\to0$ ensures that fluctuations also vanish on the accelerated time scale.

\section{Implicit regularization}\label{section : implicit regularization}

The previous section characterized how stochastic smoothing changes the transient dynamics of spectral descent. We now characterize how the same smoothing mechanism determines the interpolator selected at convergence.
For diagonal matrices, the matrix-sign operator reduces to the coordinate-wise sign. Thus, smoothened sign descent on a diagonal linear network is a direct restriction of the spectral dynamics underlying Muon.
In this setting, the induced dynamics can be written as a mirror flow (see Appendix \ref{appendix : preliminary implicit bias} for details). This allows us to characterize the implicit regularization that guides model selection. 

Before studying the more general depth $L$ case, for simplicity, we first consider the two-layer parameterization $\theta=u\odot v$, where $u,v\in\mathbb{R}^d$ and $\odot$ denotes pointwise multiplication. Given a dataset $\{(x_i,y_i)\}_{i=1}^n\subset\mathbb{R}^d\times\mathbb{R}$, we consider $ \mathcal{L}(\theta) = \frac{1}{2n}\sum_{i=1}^n \bigl(\langle x_i,u\odot v\rangle-y_i\bigr)^2.$ Unlike the whitened matrix factorization setting of Section \ref{section : main result spectral dynamics}, we do not require any assumption on the data covariance.
Under the same noise assumptions as in the previous section, the smoothened sign dynamics of $\theta_t=u_t\odot v_t$ can (up to a positive time rescaling) be written as a mirror flow. The derivation, which relies on the preserved balance $|u_t|-|v_t|=|u_0|-|v_0|$ together with the construction of the corresponding mirror potential, is presented in Appendix~\ref{appendix : implicit bias}.

\begin{theorem}\label{lemma : implicit bias two layer}
    The iterates $\theta_t = u_t \odot v_t$ converge and the loss decreases with rate $\mathcal{O}(1/t^2)$ to a solution of the following objective:
    \begin{equation*}
        \min_{\theta \in \mathbb{R}^d} R(\theta, \lambda_0) \qquad \text{ such that } \qquad X \theta = y
    \end{equation*}
    where $
R(\theta,\lambda_0)
=
\sum_{j=1}^d
\left[
\frac{1}{12}\left(4|\theta_j|+\lambda_0^2\right)^{3/2}
-\frac{|\lambda_0|}{2}|\theta_j|
-\frac{\theta_j\text{Sign}(\theta_{j,0})}{2}
\left(
\sqrt{4|\theta_{j,0}|+\lambda_0^2}-|\lambda_0|
\right)
\right] 
$
is a Legendre (Definition \ref{definition : Legendre function}) and Bregman (Definition \ref{definition : Bregman function appendix}) potential with respect to $\theta \in \mathbb{R}^d$ and $\lambda_0 : = |u_0| -|v_0|$. Moreover, if $\mathcal{L}$ satisfies with $\Lambda > 0$ the PL-inequality (Definition \ref{definition : PL inequality}), then for $t \geq T$ and $T = \sqrt{ \frac{\pi}{2}\mathcal{L}(\theta_0)} /\left( |\lambda_0| \Lambda\right)$ convergence is guaranteed. 
\end{theorem}
\textit{
Proof sketch. The balance $\lambda_t := |u_t| - |v_t| = \lambda_0$ holds almost everywhere. Thus, the dynamics of $\theta$ can be written as a mirror flow with a different time clock, characterizing the optimization problem. 
The convergence follows from a generalized Gronwall inequality (see Appendix \ref{appendix : implicit bias}).
}

The resulting potential is the diagonal counterpart of the square-root mobility obtained for the smoothened spectral dynamics in Section \ref{section : main result spectral dynamics}. Therefore, the same stochastic smoothing that restores mode-dependent learning rates also induces a non-Euclidean implicit regularization.

\begin{remark}\label{remark : two layer implicit bias}
    If $\theta_0 = 0$, then as $\lambda_0 \rightarrow 0$ we have $ \lim_{\lambda_0 \rightarrow 0}R(\theta, \lambda_0) = \frac{2}{3} \|\theta \|_{L_{3/2}}^{3/2}$. 
\end{remark}
As a consequence of Remark \ref{remark : two layer implicit bias}, small balanced initialization yields an $L_{3/2}$-type implicit regularization.
Thus, although stochastic smoothing restores dependence on the end-to-end loss gradient, the resulting dynamics induce a different mirror geometry from gradient flow. Restoring gradient-magnitude dependence therefore does not recover the typical $L_1$ bias of gradient flow \citep{woodworth2020kernel}.
Next, we extend the result to a depth-$L$ diagonal linear network $ \theta=\Pi_{j=1}^L u_j$. As in the two-layer case, the smoothened sign dynamics preserve pairwise balances between the absolute values of the factors. These balances allow all layerwise magnitudes to be expressed in terms of $\theta$ and $\lambda_0$, producing the deep mirror geometry characterized in \cite{jacobs2026never}. For the next result, we assume that the initialization satisfies the corresponding $\lambda_0$-balance across layers.

\begin{corollary}\label{lemma : deep mirror sign implicit}
    The iterates $\theta_t = \Pi_{j =1}^L u_{j,t}$ with $L \geq 2$, converge and the loss decreases with rate $\mathcal{O}(1/t^2)$ to a solution of the following objective:
    \begin{equation*}
        \min_{\theta \in \mathbb{R}^d} R(\theta, \lambda_0) \qquad \text{ such that } \qquad X \theta = y
    \end{equation*}
    where $R(\theta, \lambda_0)$ satisfies the following properties: 1) $R : \mathbb{R}^d \rightarrow \mathbb{R}$ is a Legendre and Bregman potential with respect to $\theta \in \mathbb{R}^d$, 2) if $\theta_0 = 0$, then there is a constant $c_L>0$ such that  $\lim_{\lambda_0 \rightarrow 0} R(\theta, 
    \lambda_0) =c_L \|\theta \|_{L_{1+\frac{1}{L}}}^{1+ \frac{1}{L}}$ . Moreover, if $\mathcal{L}(\theta)$ satisfies with $\Lambda > 0$ the PL-inequality (Definition \ref{definition : PL inequality}), then for $t \geq T$ and $T = \sqrt{ \frac{\pi}{2}\mathcal{L}(\theta_0)} /\left( |\lambda_0|^{(L-1)} \Lambda\right)$ convergence is guaranteed. 
\end{corollary}
Proof. See Appendix \ref{appendix : implicit bias}.

\begin{wraptable}{r}{0.5\textwidth}
\caption{Implicit regularization objective $R$ for small init $\lambda_0>0$, depending on depth $L$.}
\centering
\begin{tabular}{c|c|c}
\hline
 & $L =2$ & $L = \infty$ \\
\hline
 SignGF &  $L_{3/2}$ & $L_1$ \\
GF & $L_1$ & $L_1$ \\
\end{tabular}\label{table : induced implicit regularization comparison}
\end{wraptable}

Corollary \ref{lemma : deep mirror sign implicit} shows that depth changes the induced implicit regularization. The exponent decreases from $3/2$ at depth $L=2$ to $1$ as $L$ increases (Table \ref{table : induced implicit regularization comparison}). Hence, increasing depth strengthens the sparsity bias of the smoothened dynamics. At the same time, the potential remains convex for every finite depth, which ensures that the corresponding implicit regularization problem remains globally well defined.

\begin{remark}
    Corollary \ref{lemma : deep mirror sign implicit} stands in stark contrast with the known results for gradient flow \citep{woodworth2020kernel}, where independent of depth $L$ the implicit regularization is $L_1$ for finitely small initialization $\lambda_0$ (Table \ref{table : induced implicit regularization comparison}). Moreover, for $\lambda_0 = 0$, the corresponding gradient flow mirror potential is non-convex when $L > 2$ on $\mathbb{R}^d$, whereas the smoothened sign gradient flow mirror potential remains convex for all $L \geq 2$. 
\end{remark}

Together with Section~\ref{section : main result spectral dynamics}, this shows that stochastic smoothing restores sequential feature learning through mode-dependent growth and induces a limiting implicit regularization approaching $L_1$ with increasing depth. We now examine these effects empirically.

\section{Empirical validation}\label{section : empirical validation}
We validate the theory in controlled linear settings and study learned complexity in Transformers.

\paragraph{Matrix factorization.}
The matrix factorization experiment in Figure~\ref{fig: main} illustrates our main prediction. Full-batch Muon learns singular modes in parallel, whereas small-batch noise restores sequential feature learning. In Appendix~\ref{appendix : matrix factorization}, we provide additional ablations showing that this qualitative distinction persists when momentum is included and across a wider range of learning rates.

\begin{figure}[ht]
    \centering
    \includegraphics[width=0.99\linewidth]{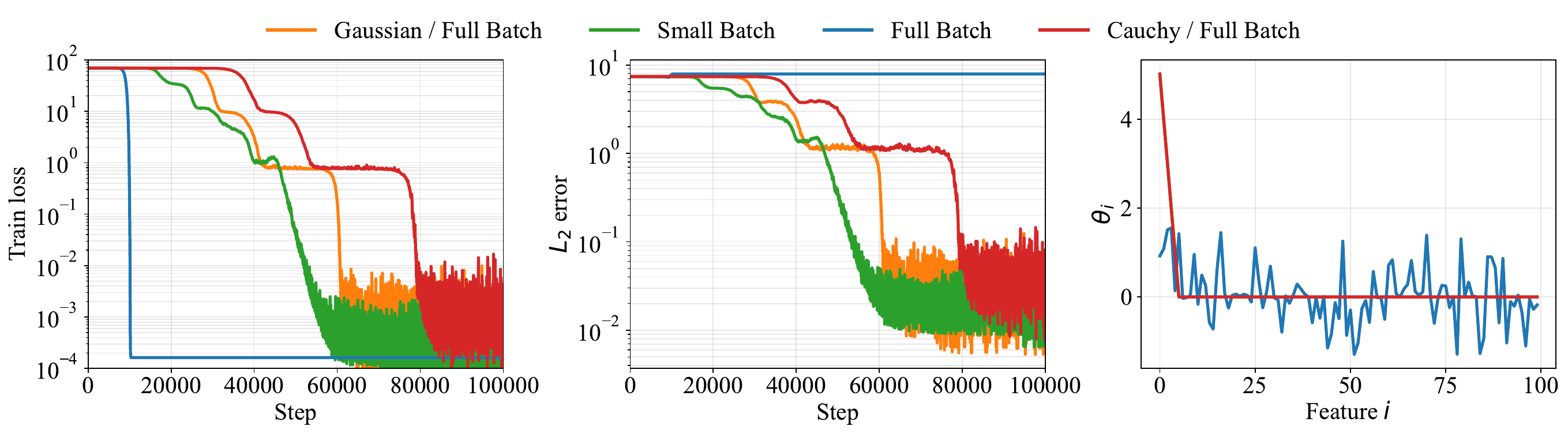}
    \caption{Noise changes the solution selected by sign descent in a diagonal linear network of depth $L =20$. We compare full-batch updates, small-batch updates, and full-batch updates with Gaussian or Cauchy noise. Left: training loss. Middle: $L_2$ reconstruction error. Right: final predictor coefficients, with first 5 non-zero. Small-batch and injected-noise runs recover the same sparse target and thus largely overlap, whereas deterministic full-batch training reaches low loss without recovery.}
    \label{fig: implicit bias and noise injection}
\end{figure}

\paragraph{Implicit regularization and heavy-tailed noise.}
We train a diagonal linear network on a sparse reconstruction task (see Appendix \ref{appendix : diagonal linear network}). Figure~\ref{fig: implicit bias and noise injection} shows that full-batch sign descent achieves low training loss but selects a dense solution with large reconstruction error (and thus poor generalization). Small-batch updates, injected Gaussian or Cauchy noise instead recover the sparse target structure. This is consistent with the predicted $L_{1+1/L}$-type implicit regularization at large depth ($L=20$) and indicates that the smoothing mechanism extends to heavy-tailed noise with infinite variance.

\begin{figure}
    \centering
     \includegraphics[width=0.8\linewidth]{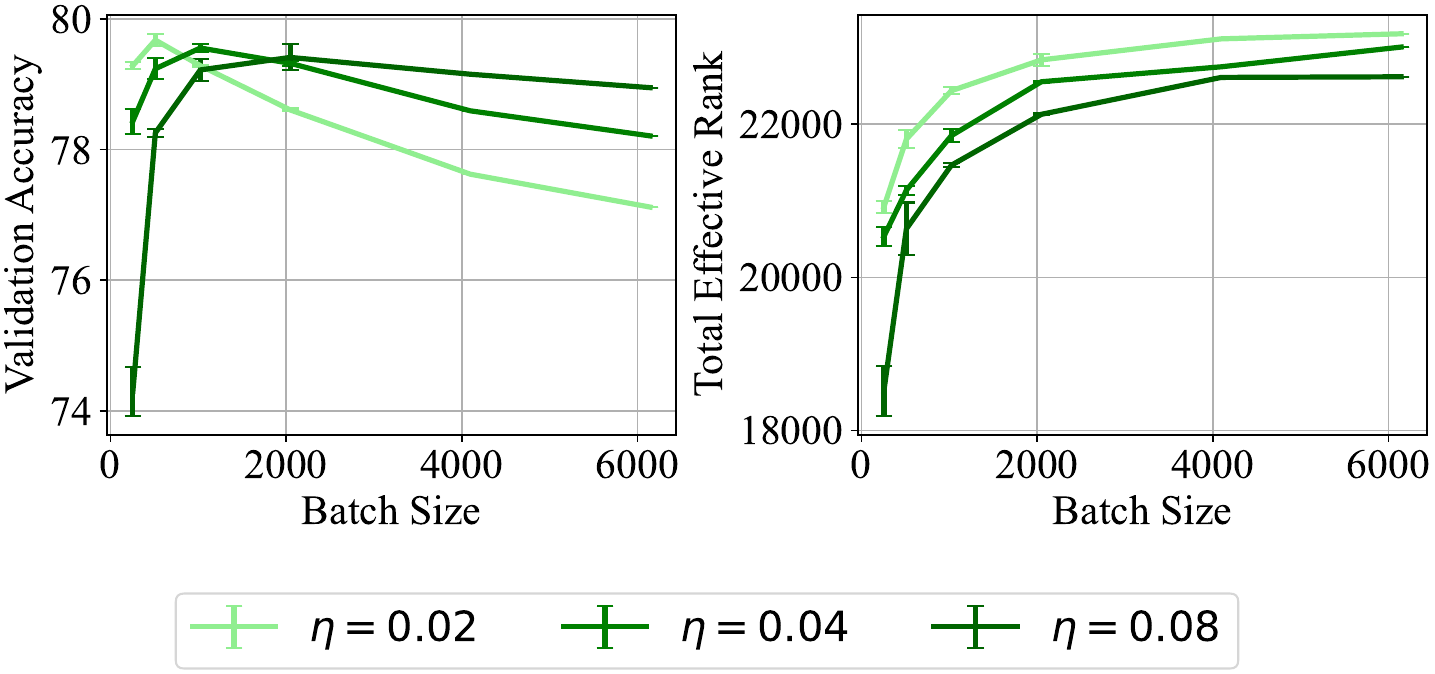}
    \caption{Final validation accuracy and effective rank of a DeiT-small model trained on ImageNet. Small batches lead to lower effective rank indicating the presence of simplicity bias.}
    \label{fig: cifar100 exp}
\end{figure}

\begin{wraptable}{r}{0.6\textwidth} 
    \centering
    \caption{Total effective ranks for a 162M-parameter Llama model trained on 3.2B tokens from FineWeb-Edu. Larger batches ($B = 256$) yield higher effective ranks under both fixed and linearly scaled learning rates.}
    \label{tab:llm-effective-rank}
   \begin{tabular}{lrr}
    \toprule
    Batch size & Fixed $\eta$ & Scaled $\eta$ \\
    \midrule
    $B$  & $42{,}927.81 \pm 17.82$ & $42{,}927.81 \pm 17.82$ \\
    $2B$ & $43{,}040.07 \pm 15.96$ & $42{,}951.67 \pm 9.63$  \\
    $4B$ & $43{,}181.65 \pm 9.73$  & $43{,}047.80 \pm 10.34$ \\
    \bottomrule
\end{tabular}
\end{wraptable}

\paragraph{Batch size, learning rate, and model complexity.}
To highlight that our general insights can transfer to transformers, we train a DeiT-small \cite{Touvron2020TrainingDI} on ImageNet with Muon, varying batch size and learning rate. Figure~\ref{fig: cifar100 exp} shows that smaller batches yield lower effective rank (Appendix \ref{appendix : vision experimental details}) at the end of training across the tested learning rates, consistent with a stronger simplicity bias, while larger batches tend to reduce validation accuracy.
Importantly, the same trend appears in the functionally relevant attention interaction $Q^\top K$, whose effective rank increases by roughly (42\%) across the batch-size range (Figure~\ref{fig:qk rank}).
We observe similar trends on CIFAR-100 and CIFAR-10 (Appendix~\ref{appendix : vision experimental details}). Moreover, we train a 162M-parameter Llama model on FineWeb-Edu at a fixed token budget. Table~\ref{tab:llm-effective-rank} shows a consistent increase in the sum of effective ranks with batch size under both fixed and linearly scaled learning rates.
Together, these results support a connection between batch size and learned complexity beyond the linear setting.

\section{Discussion and conclusion}
Our analysis highlights a qualitative difference between gradient flow and the spectral flow underlying Muon. Gradient flow already exhibits sequential feature learning in the deterministic setting and centered noise leaves its expected dynamics unchanged. 
In contrast, spectral flow removes gradient-magnitude information, causing singular modes to be learned in parallel in the full-batch setting. 
Stochasticity (resulting from mini-batch sampling) is therefore not just a perturbation of the deterministic Muon dynamics.
By smoothing the spectral sign operator, small-batch noise reintroduces gradient-dependent growth and, together with depth, recovers sequential feature learning.

This suggests that the simplicity bias of an optimizer cannot generally be inferred from its deterministic optimization geometry alone. 
In our setting, stochastic smoothing affects both the transient dynamics and the selected solution, inducing a depth-dependent $L_{1+1/L}$-type implicit regularization that approaches $L_1$ with increasing depth. Consistent with these findings, larger batches yield higher effective rank in both vision and language models, linking batch size directly to learned complexity.
More broadly, our results point to an interaction between optimizer geometry, parameterization, and stochasticity: The same deterministic update rule can exhibit qualitatively different simplicity biases depending on the training regime. 
Future extensions of our theory to structured mini-batch noise, nonlinear architectures, and other matrix-aware optimizers may further clarify when stochasticity merely perturbs an optimizer's implicit bias and when it fundamentally reshapes it.

\newpage


\paragraph{Acknowledgements.}
We gratefully acknowledge funding from the European Research Council (ERC) under the Horizon Europe Framework Programme (HORIZON) for proposal number 101116395
SPARSE-ML.

\bibliography{sample}
\bibliographystyle{iclr2027_conference}

\appendix

\newpage

\section{Extended related work}\label{appendix : extended related work}

\paragraph{Saddle escape and regularization.}
There is also work on how stochasticity, normalization, and non-Euclidean geometry affect escape from saddle points
\citep{Jin2017HowTE,roy2020escaping,Long2026DimensionFreeSE,jacobs2026never}. In particular, dimension-independent saddle-escape guarantees have recently been established for Muon \citep{Long2026DimensionFreeSE}. These results characterize local escape from saddle neighborhoods and are therefore closely connected to simplicity bias, since sequential feature learning can be viewed as progression through a hierarchy of saddles. Our work focuses on the relative escape time of each feature, which determines whether features emerge sequentially or in parallel.
Finally, optimizer geometry interacts with explicit regularization and parameterization in selecting the eventual predictor
\citep{dangelo2024needweightdecaymodern,jacobs2025mirror,jacobs2024maskmirrorimplicitsparsification}. In this work we focus however on the sequential learning and implicit regularization themselves without explicit regularization.

\section{Mirror flow and implicit bias}\label{appendix : preliminary implicit bias}

We recall the implicit regularization results for mirror flow that we use in the
proofs below. Let $X\in\mathbb{R}^{n\times d}$ denote the data matrix whose
$i$-th row is $x_i^\top$, and let $y\in\mathbb{R}^n$ denote the corresponding
labels. We consider the mean squared error loss
\[
    \mathcal{L}(\theta)
    := \frac{1}{2n}\sum_{i=1}^n
    \bigl(\langle x_i,\theta\rangle-y_i\bigr)^2.
\]
Given a Legendre function $R:\mathbb{R}^d\rightarrow
\mathbb{R}\cup\{+\infty\}$, the corresponding mirror flow is defined by
\begin{equation}\label{equation : mirror}
    \frac{d}{dt}\nabla R(\theta_t)
    =-\nabla\mathcal{L}(\theta_t),
    \qquad
    \theta_0=\theta_{\mathrm{init}}.
\end{equation}
Whenever $R$ is twice differentiable, this can equivalently be written in
primal coordinates as
\[
    \dot{\theta}_t
    =-\nabla^2R(\theta_t)^{-1}
      \nabla\mathcal{L}(\theta_t).
\]
Thus, the inverse Hessian of $R$ gives the parameter dependent mobility of
the flow. This observation will allow us to identify the smoothened sign
dynamics derived in the main text with a mirror flow.

\begin{definition}(Legendre Function, Definition 3.8 (\citep{Li2022ImplicitBO})) \label{definition : Legendre function}
Let $R : \mathbb{R}^d \rightarrow \mathbb{R} \cup \{\infty\}$ be a differentiable
 convex function. We say $R$ is a Legendre function when the following holds:
1) $R$ is strictly convex on the interior of its domain $\text{int} (\text{dom} R)$.
2) For any sequence $\{\theta_i\}^{\infty}_{i = 1}$ going to the boundary of $\text{dom} R$, the gradient diverges, i.e.  $\lim_{i\rightarrow \infty}||\nabla R(\theta_i)||_{L_2}^2 = \infty$.
\end{definition}

The geometry induced by $R$ can be expressed through its associated Bregman divergence.

\begin{definition}\label{definition : Bregman divergence}
For a differentiable convex function
$R:\mathbb{R}^d\rightarrow\mathbb{R}\cup\{+\infty\}$, the Bregman divergence
between $\theta_1,\theta_2\in\operatorname{int}(\operatorname{dom}R)$ is
defined as
\[
    D_R(\theta_1,\theta_2)
    :=
    R(\theta_1)-R(\theta_2)
    -\left\langle
        \nabla R(\theta_2),\theta_1-\theta_2
    \right\rangle.
\]
\end{definition}

The Bregman divergence is generally not symmetric, but it is non-negative
and vanishes when its two arguments coincide. For the Euclidean potential
$R(\theta)=\frac{1}{2}\|\theta\|_2^2$, it reduces to
$D_R(\theta_1,\theta_2)=\frac{1}{2}\|\theta_1-\theta_2\|_2^2$.
Moreover, adding an affine function to $R$ changes neither its Hessian nor its Bregman divergence. To guarantee convergence of the mirror flow, we additionally use the
following regularity conditions.

\begin{definition}(Bregman function Definition 4.1 \citep{Alvarez_2004})\label{definition : Bregman function appendix}
 A function $R$ is called a
 Bregman function if it satisfies the following properties:
 \begin{itemize}
     \item $\text{dom}R$ is closed. $R$ is strictly convex and continuous on $\text{dom} R$. $R$ is $C^1$ on $\text{int} (\text{dom}R ))$.
     \item For any $\theta \in \text{dom} R$ and $\gamma \in \mathbb{R}$, $\{\xi \in \text{dom} R | D_R(\theta,\xi) \leq \gamma\}$ is bounded.
     \item For any $\theta \in \text{dom} R$ and sequence $\{\theta_i\}^{\infty}_{i=1} \subset \text{int}(\text{dom} R)$ such that $\lim_{i \rightarrow \infty} \theta_i = \theta$, it holds that $\lim_{i\rightarrow \infty} D_R(\theta,\theta_i) \rightarrow 0$.
     
 \end{itemize}
\end{definition}

\paragraph{Time reparameterization.}
The dynamics derived in the main text contain an additional positive scalar
factor depending on the loss. More generally, consider
\[
    \dot{\theta}_t
    =-a_t\nabla^2R(\theta_t)^{-1}
      \nabla\mathcal{L}(\theta_t),
    \qquad a_t>0.
\]
Under the time change $\tau(t)=\int_0^t a_s\,ds$, this becomes the standard
mirror flow in Eq.~(\ref{equation : mirror}). Consequently, such a scalar
factor changes the speed at which the trajectory is traversed, but not the
trajectory itself or its limiting implicit bias.

\paragraph{Implicit regularization.}
For overparameterized linear regression, the loss can admit many
interpolating solutions. Mirror flow selects among them according to the
Bregman geometry induced by $R$. The following result makes this statement
precise.

\begin{theorem}(Theorem 3.9 \citep{Li2022ImplicitBO})\label{thm : mirror lin reg implicit bias reparams}
 Denote the dataset $\{x_i,y_i\}_{i =1}^n$ as $(X, y)$. Initialized at
 $\theta_0 = \theta_{\text{init}}$, assume that the mirror flow Eq.~(\ref{equation : mirror}) converges to $\theta_{\infty} = \lim_{t\rightarrow\infty} \theta_t$, which satisfies
 $X^T\theta_{\infty} =y$, then
 \begin{equation*}
     D_R(\theta_{\infty}, \theta_0) = \min_{\theta : X\theta =y} D_R(\theta, \theta_0), 
 \end{equation*}
 where  $D_R(\theta, \theta_0)$ is defined as in Definition \ref{definition : Bregman divergence}. 
\end{theorem}

Thus, once the effective dynamics have been written as a mirror flow, their
implicit regularization is determined by the corresponding potential $R$.
The remaining condition in Theorem
\ref{thm : mirror lin reg implicit bias reparams} is convergence. We use the
following result, stated as Theorem 4.14 in \cite{Li2022ImplicitBO} and
originally as Theorem 4.2 in \cite{Alvarez_2004}, to establish it.

\begin{theorem}(Theorem 4.14 \citep{Li2022ImplicitBO})\label{Thm : Arora 4.14}
Let $R$ be a Legendre and Bregman function, which is twice differentiable on its interior domain and $\nabla^2 R$ is locally Lipschitz. Moreover, assume that $\mathcal{L}$ is quasi-convex, $\nabla \mathcal{L}$ is locally Lipschitz and $\text{argmin} \{ \mathcal{L}(x) | x \in \mathbb{R}^n\}$ is non-empty. Then as $t \rightarrow \infty$, $\theta_t$ converges to some critical point $\theta^*$. Moreover, if $\mathcal{L}$ is convex $\theta_t$ converges to a minimizer of $\mathcal{L}$.
\end{theorem}

Our argument regarding the implicit regularization therefore consists of three steps: we first identify the
mobility of the smoothened sign dynamics and construct the corresponding mirror potential $R$; we then verify that $R$ is a Legendre and Bregman function so that Theorem \ref{Thm : Arora 4.14} guarantees convergence and finally, Theorem \ref{thm : mirror lin reg implicit bias reparams} characterizes the selected interpolating solution. Furthermore under the additional assumption of the PL-inequality (Definition \ref{definition : PL inequality}) we can get convergence in finite time for our setting. Note for the MSE loss $\mathcal{L}(\theta^*) =0$.

\begin{definition}(PL-inequality)\label{definition : PL inequality}
    For a differentiable function $\mathcal{L}: \mathbb{R}^d \rightarrow \mathbb{R}$ and positive constant $\Lambda > 0$ we say $\mathcal{L}$ satisfies the PL-inequality if for all $\theta \in \mathbb{R}^d$ we have:
    \begin{equation*}
        \|\nabla \mathcal{L}(\theta) \|_2^2 \geq 2\Lambda ( \mathcal{L}(\theta) -\mathcal{L}(\theta^*) )
    \end{equation*}
    where $\theta^* \in \argmin_{\theta \in \mathbb{R}^d} \mathcal{L}(\theta)$.
\end{definition}

\newpage

\section{Proof of the spectral flows}\label{section : expected flows}

Proof of Theorem \ref{theorem: spectral flow}.

\begin{proof}
We first show the result for $L =2$. The data are assumed to be whitened, \(XX^\top/n=I_d\), the loss can be written,
up to an additive constant independent of \(U\) and \(V\), as
\[
\mathcal{L}(U,V)
=
\frac{1}{2}\left\|UV-C_{yx}\right\|_{\mathrm F}^2,
\qquad
C_{yx}=P\Lambda Q^\top,
\]
where
\[
\Lambda=\operatorname{diag}(\lambda_1,\ldots,\lambda_d).
\]
It follows that
\[
\nabla_U\mathcal{L}(U,V)
=
(UV-C_{yx})V^\top
\]
and
\[
\nabla_V\mathcal{L}(U,V)
=
U^\top(UV-C_{yx}).
\]

By singular-vector alignment and balanced initialization, we may write
\[
U_0=PD_0,
\qquad
V_0=D_0Q^\top,
\]
where
\[
D_0
=
\operatorname{diag}
\left(
\sqrt{\theta_{1,\mathrm{init}}},
\ldots,
\sqrt{\theta_{d,\mathrm{init}}}
\right).
\]

We first show that this structure is preserved by the spectral gradient
flow. Suppose that, at some time \(t\),
\[
U_t=PD_t,
\qquad
V_t=D_tQ^\top,
\]
where
\[
D_t=\operatorname{diag}(d_1(t),\ldots,d_d(t)),
\qquad
d_i(t)\geq 0.
\]
Then
\[
U_tV_t=PD_t^2Q^\top,
\]
and hence
\[
\theta_i(t)=d_i(t)^2.
\]
Moreover,
\[
U_tV_t-C_{yx}
=
P(D_t^2-\Lambda)Q^\top.
\]
Therefore,
\[
\nabla_U\mathcal{L}(U_t,V_t)
=
P(D_t^2-\Lambda)D_t
\]
and
\[
\nabla_V\mathcal{L}(U_t,V_t)
=
D_t(D_t^2-\Lambda)Q^\top.
\]

Define
\[
S_t
=
\operatorname{diag}
\left(
\text{Sign}(\theta_i(t)-\lambda_i)
\right)_{i=1}^d.
\]
Since all matrices inside the singular-vector bases are diagonal, the
matrix sign operator satisfies
\[
\text{Sign}\left(\nabla_U\mathcal{L}(U_t,V_t)\right)
=
PS_t
\]
and
\[
\text{Sign}\left(\nabla_V\mathcal{L}(U_t,V_t)\right)
=
S_tQ^\top.
\]
Consequently, the spectral gradient flow is
\[
\dot{U}_t
=
-\text{Sign}\left(\nabla_U\mathcal{L}(U_t,V_t)\right)
=
-PS_t
\]
and
\[
\dot{V}_t
=
-\text{Sign}\left(\nabla_V\mathcal{L}(U_t,V_t)\right)
=
-S_tQ^\top.
\]
Thus, the aligned and balanced structure is preserved, and
\[
\dot{D}_t=-S_t.
\]
In particular, the singular modes evolve independently and satisfy
\[
\dot{d}_i(t)
=
\text{Sign}(\lambda_i-\theta_i(t)).
\]

Since \(\theta_i(t)=d_i(t)^2\), we obtain
\[
\dot{\theta}_i(t)
=
2d_i(t)\dot{d}_i(t)
=
2\sqrt{\theta_i(t)}
\text{Sign}(\lambda_i-\theta_i(t)).
\]
Absorbing the constant factor \(2\) into the time parametrization gives
\[
\dot{\theta}_i(t)
=
\sqrt{\theta_i(t)}
\text{Sign}(\lambda_i-\theta_i(t)).
\]

Fixing one singular mode and suppressing the index \(i\), we have
\[
\dot{\theta}_t
=
\sqrt{\theta_t}\,
\text{Sign}(\lambda-\theta_t),
\qquad
\theta_0=\theta_{\mathrm{init}}.
\]

To solve this equation, define
\[
r_t:=\sqrt{\theta_t}.
\]
Since \(\theta_t=r_t^2\), the chain rule gives
\[
\dot{\theta}_t=2r_t\dot{r}_t.
\]
Whenever \(r_t>0\),
\[
2r_t\dot{r}_t
=
r_t\text{Sign}(\lambda-r_t^2),
\]
and therefore
\[
\dot{r}_t
=
\frac{1}{2}
\text{Sign}(\lambda-r_t^2).
\]
Because \(r_t\geq 0\) and \(\lambda\geq 0\),
\[
\text{Sign}(\lambda-r_t^2)
=
\text{Sign}(\sqrt{\lambda}-r_t).
\]
Thus, \(r_t\) moves toward \(\sqrt{\lambda}\) at constant speed \(1/2\)
and remains there once it reaches \(\sqrt{\lambda}\).

Since
\[
r_0=\sqrt{\theta_{\mathrm{init}}},
\]
the solution is
\[
r_t
=
\sqrt{\lambda}
+
\text{Sign}\left(
\sqrt{\theta_{\mathrm{init}}}-\sqrt{\lambda}
\right)
\left(
\left|
\sqrt{\theta_{\mathrm{init}}}-\sqrt{\lambda}
\right|
-\frac{t}{2}
\right)_+.
\]
Finally, using \(\theta_t=r_t^2\), we obtain
\[
\theta_t
=
\left(
\sqrt{\lambda}
+
\text{Sign}\left(
\sqrt{\theta_{\mathrm{init}}}-\sqrt{\lambda}
\right)
\left(
\left|
\sqrt{\theta_{\mathrm{init}}}-\sqrt{\lambda}
\right|
-\frac{t}{2}
\right)_+
\right)^2.
\]
The dynamic description extension for higher depth $L$ follows analogous, this concludes the proof.
\end{proof}

Proof of Theorem \ref{theorem : spectral smooth flow}.

\begin{proof}
By alignment, the singular modes decouple, and Gaussian smoothing gives
\[
\mathbb{E}\left[
\text{Sign}(\lambda_i-\theta_i+\xi_i)
\right]
=
\operatorname{erf}\left(
\frac{\lambda_i-\theta_i}
{\sqrt{2\sigma^2\mathcal{L}_\tau}}
\right).
\]
Using
\[
\operatorname{erf}(x)
=
\frac{2}{\sqrt{\pi}}x+O(x^3),
\]
we obtain
\[
\frac{d\theta_i}{d\tau}
=
\sqrt{\frac{2}{\pi\sigma^2}}\,
\frac{\sqrt{\theta_i}(\lambda_i-\theta_i)}
{\sqrt{\mathcal{L}_\tau}}
+
O\left(
\frac{
\sqrt{\theta_i}\lvert\lambda_i-\theta_i\rvert^3
}{
\sigma^{3}\mathcal{L}_\tau^{3/2}
}
\right).
\]
Using the fact that, for all $i \in [d]$ and $\mathcal{L}>0$, we have
\[
\frac{|\lambda_i-\theta_i|}{\sqrt{2\mathcal{L}}} \leq 1,
\]
we obtain a uniform bound in terms of the noise strength $\sigma$. Therefore, for sufficiently large noise strength $\sigma$, the approximation holds uniformly along the trajectory. Moreover, as the learning rate $\eta \to 0$, the stochastic fluctuations vanish, while the expected drift remains. Hence, the stated time change yields, to leading order,
\[
\frac{d\theta_i}{dt}
=
\sqrt{\theta_i}(\lambda_i-\theta_i).
\]

For a fixed singular mode, set \(r_t:=\sqrt{\theta_t}\). Then
\[
2r_t\frac{dr_t}{dt}
=
r_t(\lambda-r_t^2),
\]
and hence
\[
\frac{dr_t}{dt}
=
\frac{1}{2}(\lambda-r_t^2).
\]
Assuming \(0<\theta_{\mathrm{init}}<\lambda\), separation of variables gives
\[
\frac{dr_t}{\lambda-r_t^2}
=
\frac{1}{2}\,dt.
\]
Integrating from \(0\) to \(t\), and using
\(r_0=\sqrt{\theta_{\mathrm{init}}}\), yields
\[
\int_{r_0}^{r_t}\frac{dr}{\lambda-r^2}
=
\frac{t}{2}.
\]
Since
\[
\int\frac{dr}{\lambda-r^2}
=
\frac{1}{\sqrt{\lambda}}
\operatorname{artanh}\left(\frac{r}{\sqrt{\lambda}}\right),
\]
we obtain
\[
\operatorname{artanh}\left(\frac{r_t}{\sqrt{\lambda}}\right)
=
\frac{\sqrt{\lambda}}{2}t
+
\operatorname{artanh}\left(
\sqrt{\frac{\theta_{\mathrm{init}}}{\lambda}}
\right).
\]
Applying \(\tanh\) to both sides gives
\[
r_t
=
\sqrt{\lambda}\,
\tanh\left(
\frac{\sqrt{\lambda}}{2}t
+
\operatorname{artanh}\sqrt{
\frac{\theta_{\mathrm{init}}}{\lambda}
}
\right).
\]
Finally, since \(\theta_t=r_t^2\),
\[
\theta_t
=
\lambda
\tanh^2\left(
\frac{\sqrt{\lambda}}{2}t
+
\operatorname{artanh}\sqrt{
\frac{\theta_{\mathrm{init}}}{\lambda}
}
\right).
\]
Now we can use the time change and the fact that the solution is non-decreasing from initialization. We use the alignment and that we then can express the Frobenius norm in terms of a Euclidean norm on the singular values $ \mathcal{L}_{\tau} = \frac{1}{2}\sum_{i= 1}^d (\lambda_i  - \theta_i)^2$, which gives us:
$$d\mathcal{L}_{\tau}  = -\nabla \mathcal{L}_{\tau} \sqrt{\theta_{\tau}} \nabla \mathcal{L}_{\tau}  /\sqrt{\mathcal{L}_{\tau}}  \sqrt{\frac{2}{\pi \sigma^2}} d\tau \leq -\min_i\sqrt{\theta_{i,0}} \sqrt{\mathcal{L_{\tau}}} \sqrt{\frac{2}{\pi \sigma^2}} d\tau$$
Applying a Gr\"onwall type argument gives the result.

\end{proof}

\newpage

\section{Implicit Bias}\label{appendix : implicit bias}

We now focus on a diagonal linear network, with the noise considered in the main text $N(0, \sigma \mathcal{L}I)$. This gives the leading order ODE:
\begin{equation}\label{equation : ode second}
    \begin{cases}
        d u_t = - \text{Sign}(v_t) \sqrt{\frac{2}{\pi}} \nabla_\theta \mathcal{L}(\theta_t) /\sqrt{\mathcal{L}t} dt\\
        d v_t = -   \text{Sign}(u_t)\sqrt{\frac{2}{\pi}} \nabla_\theta \mathcal{L}(\theta_t) /\sqrt{\mathcal{L}t}dt
    \end{cases}
\end{equation}

Note that we now have $\mathbb{E}[\text{Sign}(v \odot g_{\theta})] =\text{sign}(v) \odot \mathbb{E}[\text{Sign}( g_{\theta})] $ where $g_{\theta}$ is the gradient estimate. In the matrix factorization the sign on singular values does not appear as we assumed there that we start at a balanced initialization. In Lemma \ref{lemma : invariant} we derive a dynamical invariant of the flow in Equation \ref{equation : ode second} for constructing of the mirror flow dynamics.

\begin{lemma}\label{lemma : invariant}
    The flow in Eq.~(\ref{equation : ode second}) satisfies the following invariant:
    \begin{equation*}
        |u_t | - |v_t | = |u_0| - |v_0| =: \lambda _0
    \end{equation*}
\end{lemma}
Proof.
Apply the chain rule and use the flow definition in Eq.~(\ref{equation : ode second}).

Note that Lemma \ref{lemma : invariant} also has been established for full batch case \cite{jacobs2026never}.

Using Lemma \ref{lemma : invariant} we show now that the dynamics of $\theta$ follow a mirror flow by deriving the corresponding Riemannian gradient flow in Lemma \ref{lemma : mirror 2d sign}.

\begin{lemma}\label{lemma : mirror 2d sign}
    The dynamics of $\theta_t = u_t \odot v_t$ are described by a time rescaled mirror flow:
    \begin{equation*}
        d \theta_t = -\sqrt{4|\theta_t| + \lambda_0^2} \sqrt{\frac{2}{\pi}} \nabla_\theta \mathcal{L}(\theta_t) /\sqrt{\mathcal{L}t}dt
    \end{equation*}
\end{lemma}
Proof.
It follows from the chain rule that:
\begin{equation*}
    d \theta_t = -(|u_t| + |v_t|)\sqrt{\frac{2}{\pi}} \nabla_\theta \mathcal{L}(\theta_t) /\sqrt{\mathcal{L}t}dt
\end{equation*}

By definition and Lemma \ref{lemma : invariant} we can use the following identities:
\begin{equation*}
    \begin{cases}
        \theta = u \odot v \\
        \lambda_0 = |u| - |v|
    \end{cases}
\end{equation*}
This gives:
\begin{equation*}
    \sqrt{4|\theta| + \lambda_0^2} = \sqrt{4|u| |v| +(|u|-|v|)^2} = \sqrt{(|u| + |v|)^2} = |u| + |v|.
\end{equation*}
This concludes the proof. $\square$

\begin{proof}[Proof of Theorem \ref{lemma : implicit bias two layer}]
Define the a time change $\tau : = \sqrt{\frac{2}{\pi}} \int_0^t 1/\sqrt{\mathcal{L}_t}dt$. Then we get the ODE with a Riemannian metric corresponding to a Legendre and Bregman potential $R$:
\begin{equation*}
    d \theta_\tau = -\sqrt{4|\theta_\tau| + \lambda_0^2}  \nabla_\theta \mathcal{L}(\theta_\tau) d\tau.
\end{equation*}
It follows from Theorem 4.2 in \citep{Alvarez_2004} the iterates converge to a minimizer of $L$.
Moreover, the loss $L$ decreases linearly with respect to the time change $\tau$:
\begin{equation*}
     \mathcal{L}(\theta_{\tau})  \leq \frac{D_R(\theta^* , \theta_0)}{\tau}
\end{equation*}

\begin{equation*}
    d \tau (t) /d t = \sqrt{\frac{2}{\pi}} 1/\sqrt{\mathcal{L}_t} \geq \sqrt{\frac{2}{\pi}} \frac{\sqrt{\tau}}{\sqrt{D_R}}
\end{equation*}
which implies that 
\begin{equation*}
    \tau \geq \frac{1}{2\pi}\frac{t^2}{D_R}
\end{equation*}
Thus 
\begin{equation*}
     \mathcal{L}(\theta_{t})  \leq \frac{B}{t^2}
\end{equation*}
where $B : = 2\pi D_R^2$.

Moreover, under the PL-inequality (Definition \ref{definition : PL inequality}) we have
\begin{align*}
    d \mathcal{L}(\theta_t) &= -\nabla_\theta \mathcal{L}(\theta_t)\sqrt{4|\theta_t| + \lambda_0^2} \sqrt{\frac{2}{\pi}} \nabla_\theta \mathcal{L}(\theta_t) /\sqrt{\mathcal{L}_t}dt \\ &\leq - |\lambda_0| \sqrt{\frac{2}{\pi}} ||\nabla_{\theta} \mathcal{L}(\theta_t) ||_2^2 / \sqrt{\mathcal{L}(\theta_t)} dt \\ &\leq - 2|\lambda_0| \sqrt{\frac{2}{\pi}} \Lambda \sqrt{\mathcal{L}(\theta_t)} dt  \\
\end{align*}
Therefore by a generalized Grownwall bound we have:
\begin{equation*}
    \mathcal{L}(\theta_t) \leq \left(\sqrt{\mathcal{L}(\theta_0)} - |\lambda_0| \sqrt{\frac{2}{\pi}} \Lambda t \right)^2_+.
\end{equation*}

Finally, the KKT condition follows from the mirror flow construction. Concretely, the flow for $\theta_t$ is equivalent to a mirror flow with $\nabla_{\theta} R(\theta)$ such that:
\begin{equation*}
    d \nabla_{\theta} R(\theta_t) = - \nabla_{\theta} L (\theta_t) dt, \qquad \theta_0 = \theta_{\text{init}},
\end{equation*}
by definition of the loss and integrating over time we have that for all $t \geq 0$
\begin{equation*}
    \nabla_{\theta} R (\theta_t) = - X^T\int_0^t (X\theta_s - y) ds \in \text{span} (X),
\end{equation*}
which is exactly the KKT condition of the optimization problem.
\end{proof}

\begin{proof}[Proof of Corollary~\ref{lemma : deep mirror sign implicit}]
The argument follows the proof of Theorem~\ref{lemma : implicit bias two layer}. The
smoothened sign dynamics preserve the pairwise balances
$|u_{j,t}|-|u_{1,t}|=\lambda_0$ almost everywhere. Consequently, after the
same time change as in Theorem~\ref{lemma : implicit bias two layer}, the product
$\theta_t=\Pi_{j=1}^L u_{j,t}$ satisfies
\[
    \frac{d\theta_t}{dt}
    =
    -M_{L,\lambda_0}(\theta_t)\nabla\mathcal{L}(\theta_t),
\]
where $M_{L,\lambda_0}$ is the separable inverse metric induced by these
balances. This is precisely the $q=1$ metric characterized in
\citet[Theorem~4.8]{jacobs2026never}. Since
$(L-1)/L<1$, the corresponding potential $R(\theta,\lambda_0)$ is a Legendre
and Bregman function for every $L\geq2$. The mirror-flow implicit-bias and
convergence results from Appendix~\ref{appendix : preliminary implicit bias} therefore apply. Repeating the argument
of Theorem~\ref{lemma : implicit bias two layer} gives
$\mathcal{L}(\theta_t)=\mathcal{O}(t^{-2})$ and shows that the limiting
interpolator minimizes $R(\theta,\lambda_0)$ subject to $X\theta=y$.

For $\lambda_0\rightarrow0$, the explicit characterization in
\citet[Lemma~4.13]{jacobs2026never}, specialized to $q=1$, gives, up to
irrelevant affine and positive multiplicative terms,
\[
    R(\theta,0)
    \propto
    \sum_{i=1}^d|\theta_i|^{1+1/L}
    =
    \|\theta\|_{1+1/L}^{1+1/L}.
\]
Finally, \citet[Corollary~4.9]{jacobs2026never} gives the coercivity bound
$M_{L,\lambda_0}(\theta)\succeq
\lambda_0^{L-1}I$. Under the PL inequality (Definition \ref{definition : PL inequality}), the same differential inequality
as in Theorem~\ref{lemma : implicit bias two layer} therefore yields
\[
    T
    =
    \sqrt{\frac{\pi}{2}}\,
    \frac{\sqrt{\mathcal{L}(\theta_0)}}
         {\lambda_0^{L-1}\Lambda}.
\]
\end{proof}

\newpage

\section{Flow simulation details}
To illustrate that our derived flow description captures the underlying dynamics we simulate the flows derived in the main text.
We set the non-zero singular values to $\lambda : = [6.0, 1.2, 0.8, 0.3, 0.1]$.

We set $L =2$ in the first experiment and later compare $L =2, 20$ for the smoothened flow to illustrate the sequential feature learning emerging from depth.
For $L =2$ we can make use of the explicit solutions derived and we use them for the smoothened and smoothened delayed flows.
The initialization scale is $\theta_{\text{init}} = 1e-4$.
For the other flows, we use Runge-Kutta 45 numerical integrator from the SciPy library to simulate the flow \citep{2020SciPy-NMeth}.

We observe in Figure \ref{fig:all flows} that gradient flow indeed exhibits sequential feature learning and that the spectral flow learns all features in parallel. For the smoothened flow we observe that there is a growth separation. However, all features are learned at the same time.  
In the main text we also consider the smoothened flow and compare depths in Figure \ref{fig:depth smoothened}, where do observe sequential feature learning emerge for higher depth.

\begin{figure}
    \centering
    \includegraphics[width=0.99\linewidth]{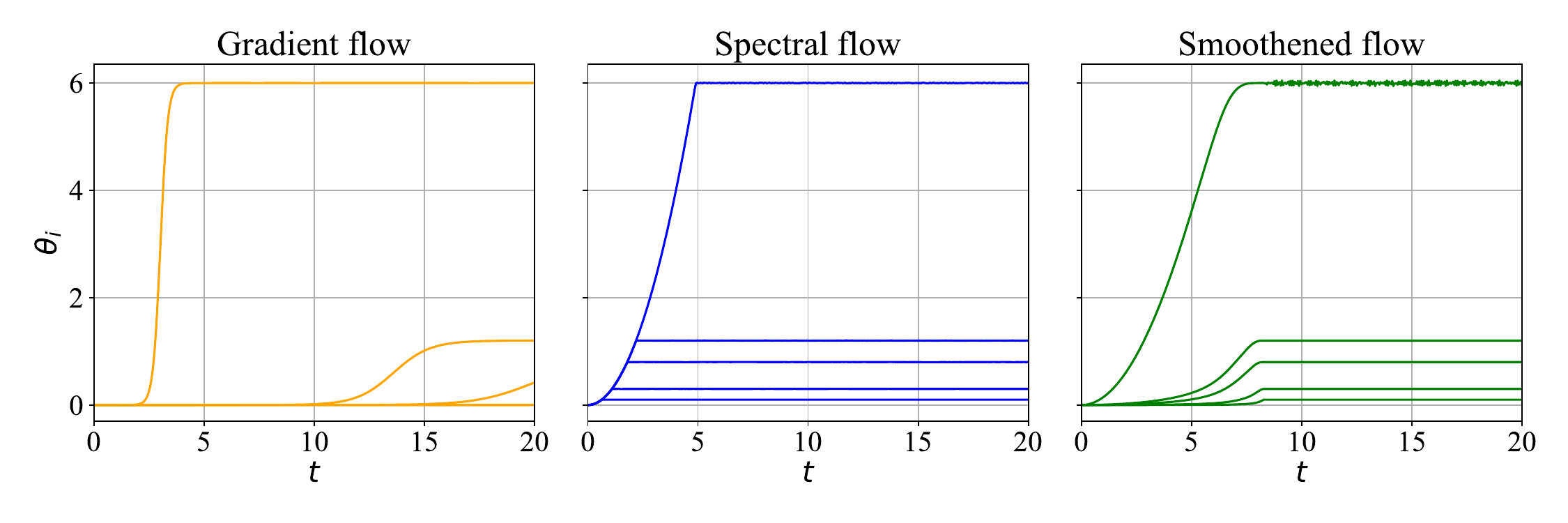}
    \caption{Simulated flows for all $3$ settings at depth $L =2$: gradient, spectral, smoothened flow.}
    \label{fig:all flows}
\end{figure}

\section{Matrix factorization details}\label{appendix : matrix factorization}

Our theoretical analysis studies a simplified matrix factorization problem. To show that our flows derived capture the dynamics in this matrix factorization problem we run the actual matrix factorization problem with small batch gradients and full batch gradients. We run gradient descent and Muon with learning rate  $\eta = 1e-2$. For all this toy experiment we set momentum $\beta = 0$. We run the experiments for $50000$ steps. 
We use whitened data, we set the dimension to of the matrix to $d =10$ and use $n =2000$ data points. The rank of the teacher model is set to $5$. 
We consider for gradient descent and Muon full batch and for Muon in addition the small batch setting with one sample. Note we still run of $50000$ steps not epochs. We use a small initialization, with $U \sim \alpha N(0,I)$ and $V \sim \alpha N(0, I)$ where $\alpha = 1e-7$ for three-way comparison and for the depth comparison it $\alpha =1e-4$.
Finally, we track in Figure \ref{fig: main} in the main text the evolution of the $5$ non-zero singular values. 

Furthermore, we conduct an additional experiment to include momentum and keeping the other hyperparameters the same. The results are presented in Figure \ref{fig: momentum}.
Moreover, we conduct an experiment for varying the learning rate and show that this does not significantly alter the flow as demonstrated in Figure \ref{fig: time lr}.

\begin{figure}
    \centering
    \includegraphics[width=0.9\linewidth]{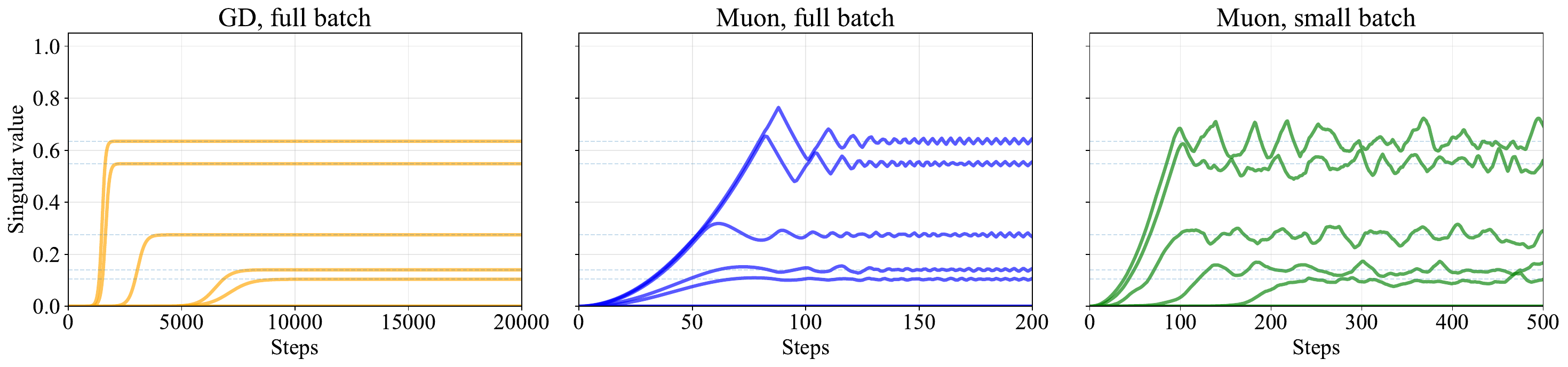}
    \caption{Gradient descent, Muon with small and large batches. The sequential learning phenomenon persists with momentum included.}
    \label{fig: momentum}
\end{figure}

\begin{figure}
    \centering
    \includegraphics[width=0.9\linewidth]{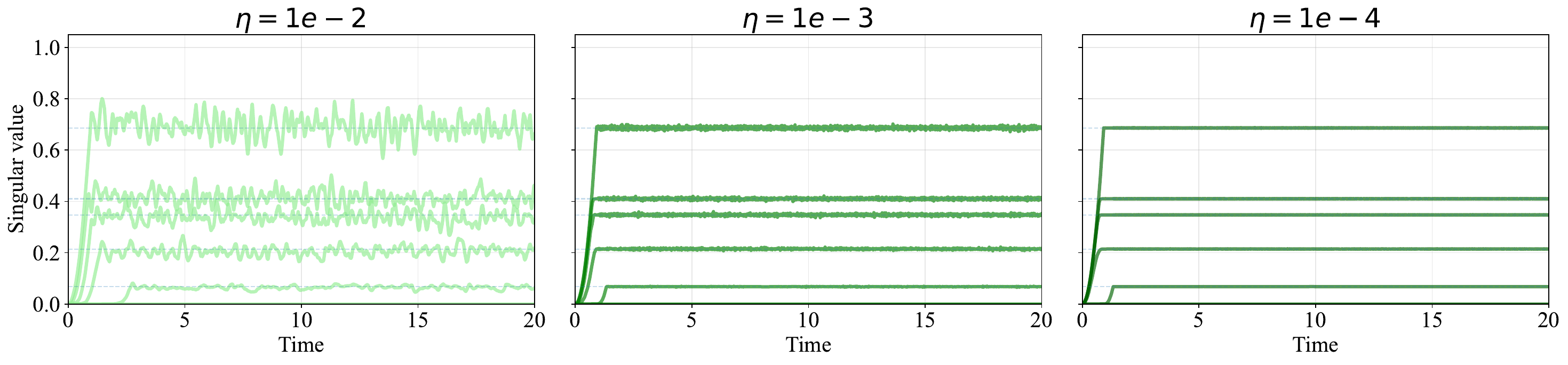}
    \caption{Matrix factorization varying the learning rate for small batch Muon. We observe that the trajectories remain similar to each other accounting for discretization time.}
    \label{fig: time lr}
\end{figure}

\section{Vision and language experimental details}\label{appendix : vision experimental details}
To understand the effect of batch size we run standard training recipes for Muon on a small DeiT for CIFAR-10.
The codebase implementation is based on \cite{Nelaturu_TurboPrune_High-Speed_Distributed}.
We conducted an exhaustive grid search over the learning rate and batch size and each configurations is ran for $3$ seeds for which we report the mean and standard $95\%$ confidence intervals. The batch size values in the sweep refer to the batch size used on each individual GPU and were selected from
\[
\{32, 64, 128, 256, 512, 1024, 2048, 4096\}.
\]
All experiments were conducted using distributed training across 4 GPUs, resulting in effective global batch sizes ranging from $128$ to $16384$ samples. The learning rate was independently varied over
\[
\eta \in \{0.005, 0.01, 0.02\},
\]
yielding a total of $24$ training configurations. Each configuration was trained for $300$ epochs using the same model architecture, optimizer (Muon), scheduler (TriangularSchedule with 10 warmup steps), weight decay set to $0.01$, CutMix, and MixUp. 
For CIFAR 100, we used the exact same configuration sweep and hyperparameters.

For ImageNet we train a DeiT-small and use Label smoothening for the loss. The batch size was set to 
\begin{equation*}
   B \in \{64, 128, 256, 512, 1024, 1536 \}*4.
\end{equation*}
The learning rate was set to 
\begin{equation*}
    \eta \in \{0.02, 0.04,0.08 \}.
\end{equation*}
Moreover, we only run one seed for the two largest batch configurations i.e. per GPU: $1024$ and $1536$. 

The results for CIFAR-10 and CIFAR-100 are reported in Figures \ref{fig: cifar10 exp} and \ref{fig: imagenet exp}. We observe that simplicity bias correlates with better performance overall.

\paragraph{Language experiment.}
We conduct the same experiment in the language domain. We consider the pretraining a LLama-type model \cite{Touvron2023LLaMAOA} on the FineWeb-Edu data set \cite{Penedo2024TheFD} for $3.2B$ tokens. For this experiment we use the codebase \cite{semenov2025benchmarking}.
As in the vision experiment we vary the batch size and learning rate and run for $3$ seeds for which we report the mean and standard $95\%$ confidence intervals.

The model we consider has $12$ layers, embedding dimension $768$, $12$ heads, and untied embeddings.
We train using D-Muon \cite{Chen2026DMuonED} with a WSD learning-rate schedule and sequence length $512$. We vary the batchsize but we also use a fixed gradient accumulation of $32$. This amounts to $32/4 = 8$ gradient accumulations for each A100 GPU.
For training, we use a base learning rate of $\eta =1e-3$, momentum parameters \(\beta_1=0.8\) and \(\beta_2=0.999\) for the additional modules trained with Adam \cite{kingma2017adammethodstochasticoptimization}, gradient clipping at \(0.5\).

For the batch size, including the gradient accumulation, we use $B \in \{256, 512, 1024\}$.
We adapt the amount of iterations to the batch size to match the amount of tokens seen. For $B =1024$ we use $6104$.

In Table \ref{tab:llm-val-acc} we report the corresponding validation accuracy of the reported effective ranks in the main text. We observe that there is a slight performance degradation when keeping the learning rate fixed and increasing the batchsize while rescaling it keeps the performance similar.

\begin{table}
    \centering
    \caption{Validation accuracy for a 162M-parameter Llama model trained on 3.2B tokens from FineWeb-Edu ($B = 256$). Results report means and approximate 95\% confidence intervals over three seeds.}
    \label{tab:llm-val-acc}
    \begin{tabular}{lcc}
        \toprule
        Batch size & Fixed $\eta$ & Scaled $\eta$ \\
        \midrule
        $256$  & $40.79 \pm 0.07$ & $40.79 \pm 0.07$ \\
        $512$  & $40.48 \pm 0.06$ & $40.71 \pm 0.09$ \\
        $1024$ & $40.18 \pm 0.02$ & $40.78 \pm 0.05$ \\
        \bottomrule
    \end{tabular}
\end{table}

\paragraph{Effective rank.}
We furthermore tracked the effective rank of each linear layer by computing it singular values $\sigma$, and numerical stability constant $\epsilon =1e-12$:
\begin{equation*}
    p = \frac{\sigma}{\sum_i \sigma_i}, \qquad \text{EffRank}(p): = \exp (- \sum_ip_i\ \text{log}(\ p_i + \epsilon) ).
\end{equation*}
In the main text we report the total sum of the effective ranks of each linear layer of the model. For the Imagenet experiment we also report the effective rank of $Q^\top K$ in Figure \ref{fig:qk rank}, where we observe the same pattern.

\begin{figure}
    \centering
    \includegraphics[width=0.5\linewidth]{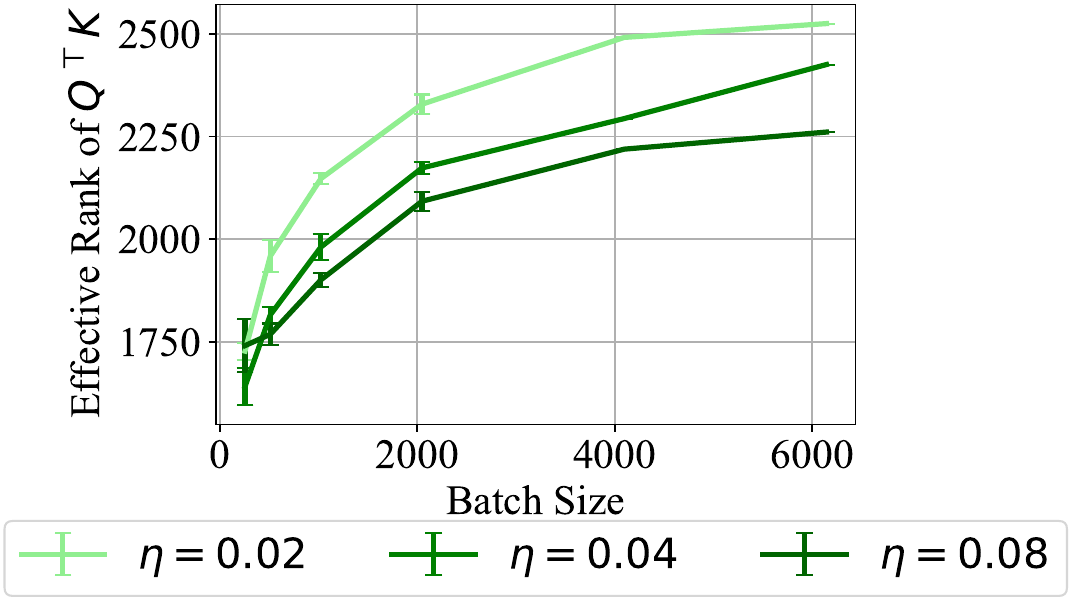}
    \caption{Effective rank of $QK$ for the DeiT-small model trained on ImageNet.}
    \label{fig:qk rank}
\end{figure}

\begin{figure} 
    \centering    \includegraphics[width=0.95\linewidth]{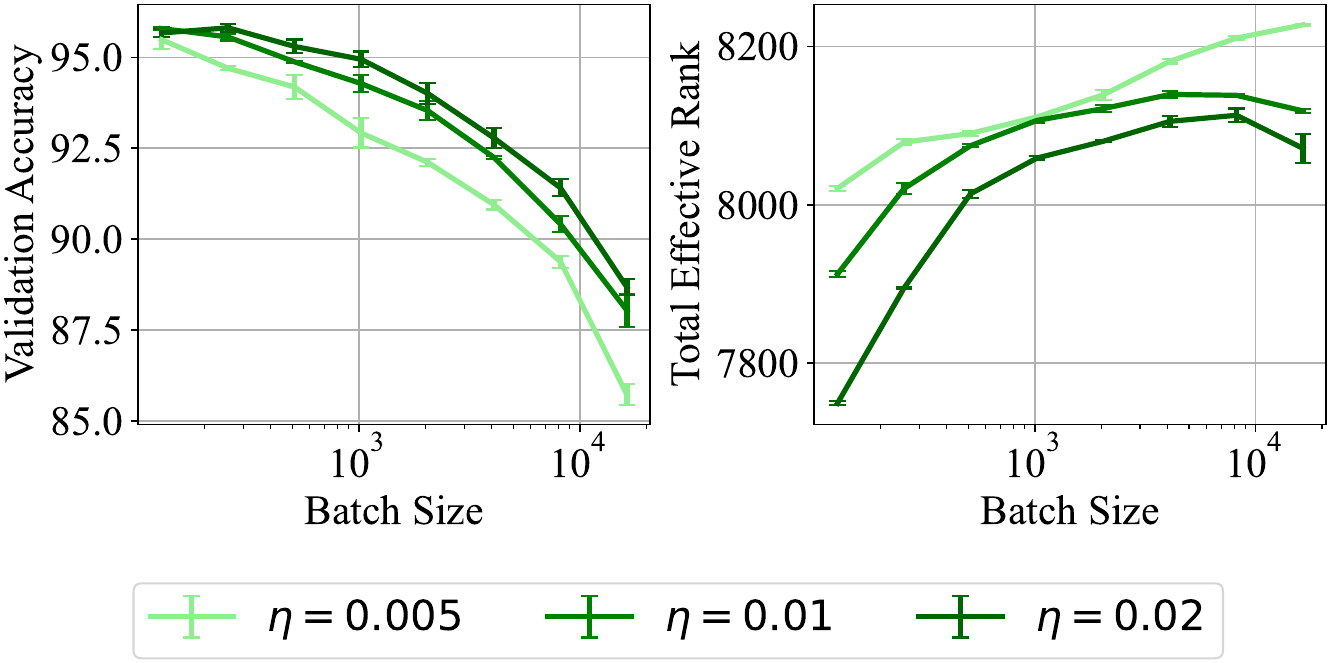}
    \caption{Final validation accuracy and effective rank of a small DeiT model trained on CIFAR 10. Small batches lead to lower rank indicating the presence of simplicity bias.}
    \label{fig: cifar10 exp}
\end{figure}

\begin{figure} 
    \centering    \includegraphics[width=0.95\linewidth]{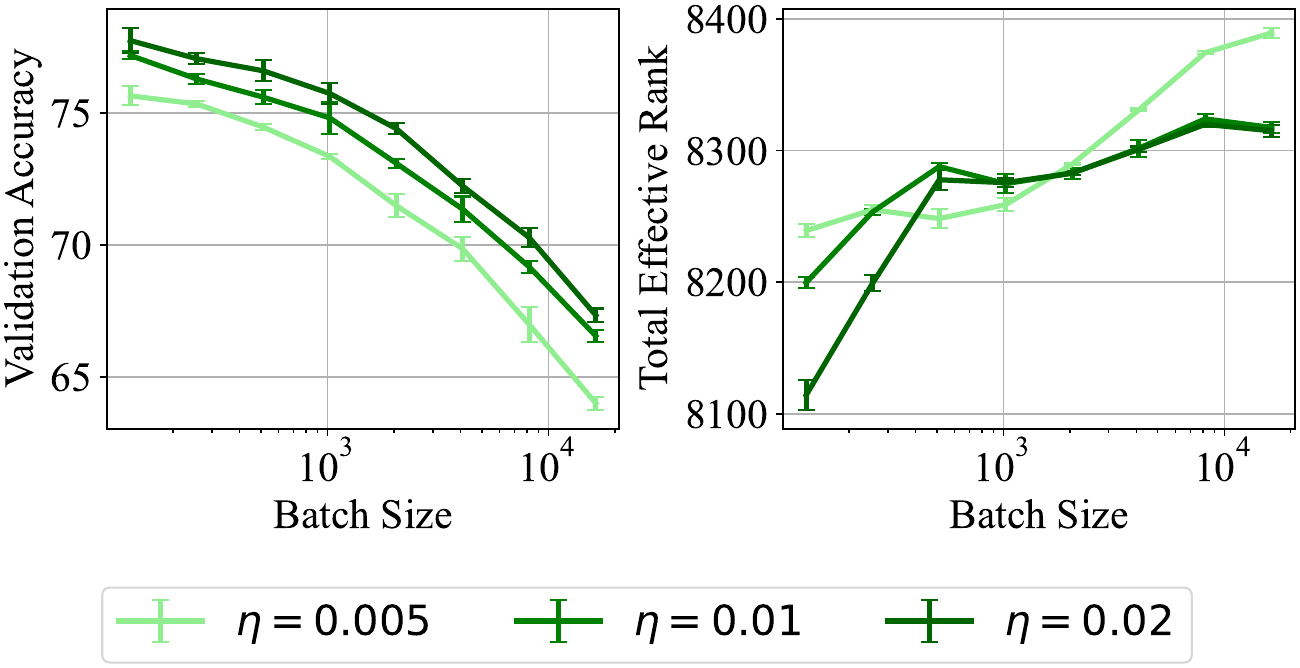}
    \caption{Final validation accuracy and effective rank of a small DeiT model trained on CIFAR-100. Small batches lead to lower rank indicating the presence of simplicity bias.}
    \label{fig: imagenet exp}
\end{figure}

\section{Diagonal linear network}\label{appendix : diagonal linear network}

In this section we provide details for the diagonal linear network experiment. We consider two depth configurations $L =2,20$. Moreover we compare sign descent with gradient descent. We also incorporate noise injections both Normal and Cauchy. The noise is added directly to the gradient with respect to $\theta$ i.e. the shared gradient.
We sample random normally distributed data with dimension $d =100$ and $n = 40$ samples. The ground truth is set to a sparse vector $\theta = (5,4,3,2,1, 0, \hdots , 0)$.

For sign descent and depth $L = 20$, we use initialization scale $\lambda_0 =1.11e-2$, learning rate $\eta =1e-4$ and train for iterations $T =100000$. For small batch, we set the batch to $1$ and the noise injection is done by sampling i.i.d. from the respective noise distributions and scaling with $\sigma =3 \sqrt{f}$, where $f$ is the mean squared error loss. The result is reported in the main text in Figure \ref{fig: implicit bias and noise injection}

For sign descent and depth $L = 2$, we use initialization scale $\lambda_0 =1.11e-3$, learning rate $\eta =1e-5$ and train for iterations $T =2000000$. For small batch, we set the batch to $1$ and the noise injection is done by sampling i.i.d. from the respective noise distributions and scaling with $\sigma =3 \sqrt{f}$, where $f$ is the mean squared error loss. The result is reported in Figure \ref{fig:dept 2 sign}. We observe that the noisy configurations reach a simpler solution than full batch but do not recover the exact solution. This is in line with an $L_{3/2}$ bias.

For gradient descent and depth $L = 2$, we use initialization scale $\lambda_0 =1.11e-3$, learning rate $\eta =1e-5$ and train for iterations $T =2000000$. For small batch, we set the batch to $1$ and the noise injection is done by sampling i.i.d. from the respective noise distributions and scaling with $\sigma =3 \sqrt{f}$, where $f$ is the mean squared error loss. The result is reported in Figure \ref{fig:dept 2 gd}. Observe that the heavy tail noise injection made the dynamics unstable. The other configurations exhibit the expected gradient flow saddle-to-saddle dynamics \cite{pesme2023saddletosaddle}.

Note that training gradient descent for high depth takes an extraordinary amount of time due to the mobility being super linear.

\begin{figure}
    \centering
    \includegraphics[width=0.9\linewidth]{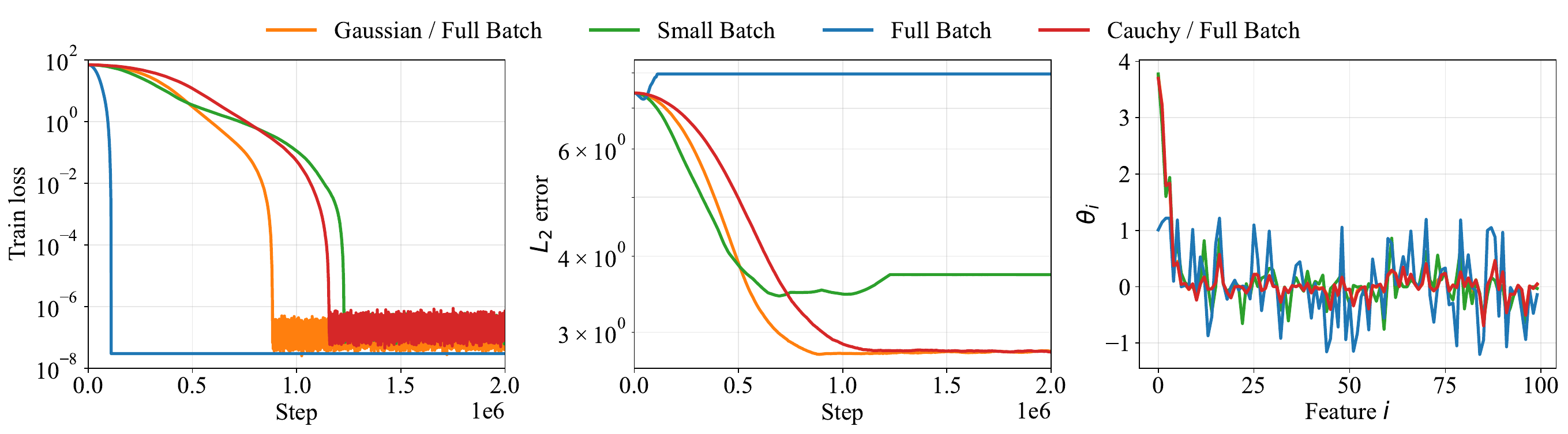}
    \caption{Noise changes the solution selected by sign descent in a diagonal linear network of depth $L =2$. We compare full-batch updates, small-batch updates, and full-batch updates with Gaussian or Cauchy noise. Left: training loss. Middle: $L_2$ reconstruction error. Right: final predictor coefficients, with first 5 non-zero. Small-batch training and both forms of injected noise recover the sparse target better but not full which is in line with the $L_{3/2}-$bias, whereas deterministic full-batch training achieves low training loss without recovery.}
    \label{fig:dept 2 sign}
\end{figure}

\begin{figure}
    \centering
    \includegraphics[width=0.9\linewidth]{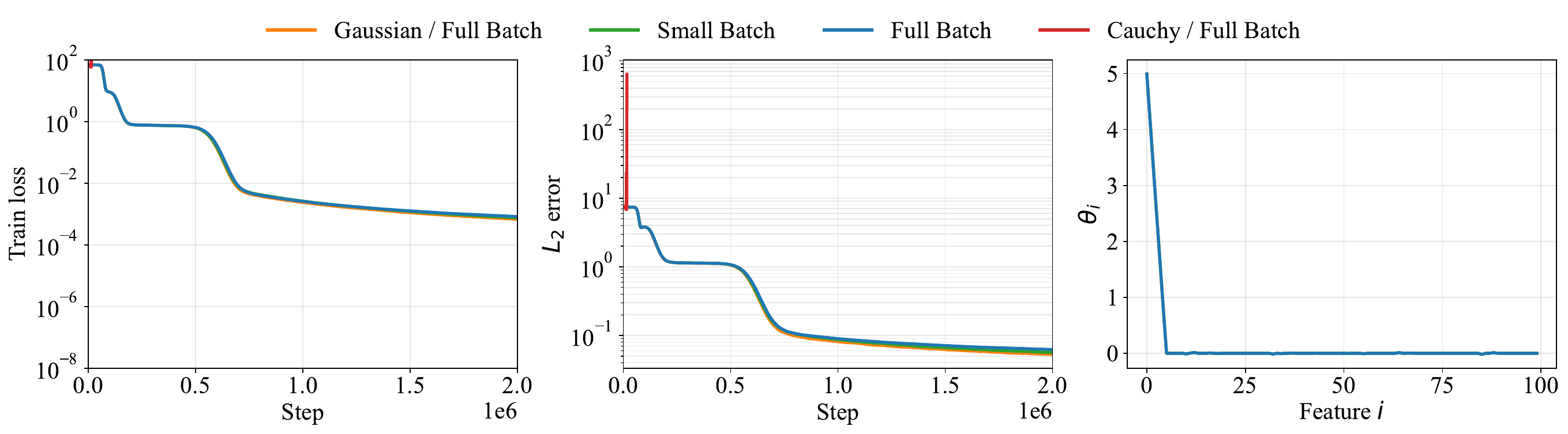}
    \caption{Noise does not change the solution selected by gradient descent in a diagonal linear network of depth $L =2$ for small initialization. We compare full-batch updates, small-batch updates, and full-batch updates with Gaussian or Cauchy noise. Left: training loss. Middle: $L_2$ reconstruction error. Right: final predictor coefficients, with first 5 non-zero. All recover the ground truth solution except the injected heavy tail noise due to too large variance.}
    \label{fig:dept 2 gd}
\end{figure}

\end{document}